%% file: main.tex
\documentclass{article}
\usepackage{kr}

\usepackage{times}
\usepackage{soul}
\usepackage{url}
\usepackage[hidelinks]{hyperref}
\usepackage[utf8]{inputenc}
\usepackage[small]{caption}
\usepackage{graphicx}
\usepackage{amsmath}
\usepackage{amsthm}
\usepackage{booktabs}
\usepackage{algorithm}
\usepackage{algorithmic}
\usepackage{amssymb}    
\usepackage{mathtools}  
\usepackage[table]{xcolor} 
\usepackage{tikz}          
\usepackage{diagbox}       
\usepackage{upgreek}    
\usepackage{stmaryrd}   
\usepackage{hhline}

\newtheorem{example}{Example}
\newtheorem{theorem}{Theorem}
\newtheorem{proposition}{Proposition}     
\newtheorem{definition}{Definition}     
\newtheorem{lemma}{Lemma}               
\newtheorem{corollary}{Corollary}       
\newcommand{\ff}{\mbox{\textit{ff}}}    
\newcommand{\comment}[1]{}              

\DeclareRobustCommand{\DirectNESS}{(\tikz[baseline=-\the\dimexpr\fontdimen22\textfont2\relax,inner sep=0pt] \draw[dash pattern={on 4.5pt off 4.5pt}](0,0) -- (5mm,0);)}
\DeclareRobustCommand{\NESS}{(\tikz[baseline=-\the\dimexpr\fontdimen22\textfont2\relax,inner sep=0pt] \draw[dash pattern={on 0.84pt off 2.51pt}](0,0) -- (5mm,0);)}
\DeclareRobustCommand{\actual}{(\tikz[baseline=-\the\dimexpr\fontdimen22\textfont2\relax,inner sep=0pt] \draw[line width=0.75](0,0) -- (5mm,0);)}

\title{Integrating Temporality and Causality into Acyclic Argumentation Frameworks using a Transition System}

\author{%
Yann Munro$^{*}$\and
Camilo Sarmiento$^{*}$\and
Isabelle Bloch\and
Gauvain Bourgne \and
Marie-Jeanne Lesot \\
\affiliations
Sorbonne Université, CNRS, LIP6, Paris, France\\
\emails
\{firstname.surname\}@lip6.fr
}

\begin{document}

    \maketitle
    {\let\thefootnote\relax\footnotetext{$^*$ These authors contributed equally to this work.}}
    
    \input{0.1_abstract}
    
    \input{0.2_introduction}

    \input{1_Argumentation_Framework}

    \input{2_Action_Language}

    \input{3_Transformation}
    \input{4_Discussion}

    \input{0.3_conclusion}
\paragraph*{Acknowledgements} The authors would like to thank Professor Catherine Adamsbaum, pediatric radiologist, for the enlightening discussions about the examples. This work was partly supported by the 3rd author’s chair in Artificial Intelligence (Sorbonne
Universit\'e and SCAI).

    \input{0.4_bibliographie}

\end{document}

%% file: 0.1_abstract.tex
\begin{abstract}
    In the context of abstract argumentation, we present the benefits of considering temporality, i.e. the order in which arguments are enunciated, as well as causality. We propose a formal method to rewrite the concepts of acyclic abstract argumentation frameworks into an action language, that allows us to model the evolution of the world, and to establish causal relationships between the enunciation of arguments and their consequences, whether direct or indirect. An Answer Set Programming implementation is also proposed, as well as perspectives towards explanations.
\end{abstract}

%% file: 0.2_introduction.tex
\section{Introduction}
    
    The abstract argumentation framework (AAF), first introduced in~\cite{dung_acceptability_1995}, provides a suitable framework for representing and reasoning about contradicting information. 
    It makes it possible to find sets of arguments that can be accepted together and provides explanations on why such sets have been accepted or not. Thus, AAF provides convenient tools to model and reason about debates. 
    However, it is a static framework which does not include a notion of temporality that seems crucial for modelling dialogues. Several types of approaches have been proposed to solve this problem. One type modifies the argumentation graph by adding or deleting attacks and arguments using specific operators~\cite{doutre2017dynamic}, and amounts to considering an AAF at each time step. Another proposal is to transform an argumentation system into a logical formalism and then use revision or belief change operators to update the argumentation system~\cite{de2016argumentation}. We propose to use another logical formalism, action languages, to model the dynamics of a dialogue.

    On the other hand, action languages offer tools to reason about action and change and have been naturally conceived to include the notion of time. The action language introduced by~\cite{sarmiento_action_2022-1}
    has been designed to determine the evolution of the world given a set of actions corresponding to deliberate choices of the agent, the occurrence of which can trigger a chain reaction through external events. We choose this action language for three reasons. First, it allows for concurrency of events. Other languages also offering this advantage, such as~$\mathcal{C}$~\cite{giunchiglia_action_1998} or PDDL+~\cite{fox_modelling_2006},
    are adapted to non-deterministic or durative actions, which increases complexity and is not useful in our framework. Secondly, there exists a definition of actual causality that is suitable for this action language. 
    Finally,~\cite{sarmiento_action_2023} propose a  sound and complete translation into ASP. 
    We propose to take advantage of these properties to  study the causal relations in a dialogue, paving the way for the search of explanations.

    This work constitutes a first step of a proposition to bring together different existing tools from the Knowledge Representation and Reasoning field to tackle issues that arise in abstract argumentation when considering causality and temporality together. This paper is structured as follows. Section \ref{sec:AAF} briefly recalls the principles of  
    abstract argumentation. Section~\ref{sec:Action_language} 
    describes the chosen action language and the actual causality definition suitable for it. Section~\ref{sec:translation} proposes the main contributions of this paper: a formalisation of acyclic abstract argumentation graphs into an action language, with its corresponding implementation in ASP.
    Section~\ref{sec:formalProp} establishes its formal properties, 
    including its soundness and completeness, as well as the relevance of the temporality inclusion. Section~\ref{sec:discussion} illustrates the exploitation of the proposed formalisation to get enriched information, as graphical representations and causal relations. Section~\ref{sec:conclusion}  concludes the paper. 

%% file: 1_Argumentation_Framework.tex
\section{Abstract Argumentation Framework, AAF}
\label{sec:AAF}

This section briefly recalls the basics of~\cite{dung_acceptability_1995}'s AAF.

\begin{definition}\label{def:AAF}
    An \emph{abstract argumentation framework}~$AAF$ is a couple~${AF = (A,R)}$, where~$A$ is a finite set of arguments and~$R$ is a set of \emph{attacks} corresponding to a binary relation~$A \times A$. An argument~$x\in A$ attacks~$y\in A$ if~$(x,y)\in R$.
\end{definition}

As $R$ is a binary relation with a finite support, an AAF can be represented using a graph.

\begin{example}
\label{ex:IRM_ou_radio}

To illustrate these notions, we introduce an argumentative scenario modelling the interaction between a requesting physician, D, and a radiologist, R, concerning an examination of a $n$ month old baby for pathology Z.
\\ \textit{D:} Can you do an X-ray scanner (CT) for this baby? ($a$) \\
\textit{R:} It is better for a baby to avoid ionising radiations. ($b$) \\
\textit{R:} I can suggest an MRI (magnetic resonance imaging) in two days' time. ($c$) \\
\textit{D:} Can Z be seen on an MRI? ($d$) \\
\textit{R:} Yes, of course! If you want confirmation, look at the guide to good radiology practice. ($e$) \\
\textit{D:} But a baby might move and so you might not be able to get the information you are looking for because the image may be artefacted. ($f$) \\
\textit{R:} Do not worry, I am used to doing MRI for babies. ($g$) \\
\textit{D:} Does not it cost the hospital a lot more to do an MRI?~($h$) I also have to check with the patient's family because it might cost them more. ($i$) \\
\textit{R:}  No problem here. The high cost includes the experience gained by my team so that in the future this kind of delicate examination can be performed without me. ($j$) \\
\textit{D:} I have just spoken to the family, no problem with the MRI, the exam is refunded.~($k$) \\
\textit{D:} 
However, the family is not comfortable with the idea of having to wait two days, could not you do the exam before?~($l$)\\
\textit{R:} No my schedule for today is already full. My next slot is in two days, as I told you.~($m$)

After this discussion, the decision is made to schedule an MRI in two days' time. But later that day, the doctor receives a call from the family saying that the baby is really not well and insisting on the urgency of the examination. Therefore, the doctor contacts the radiologist to add a final argument.\\
\textit{D:} It is very urgent for the baby, we need a place today! ($n$)

From this dialogue, we can extract manually arguments and their relations to create the AAF represented in Figure~\ref{fig:ex_radio} with the following arguments:~$\{ \textbf{a}$: Scanner, $b$: Ionising radiation, $\textbf{c}$: MRI in two days, $d$: Z not visible by MRI, $e$: Z visible by MRI, $f$: Difficult conditions, $g$: High experience, $h$: High cost for the hospital, $i$: High cost for the patient, $j$: Not problematic for the hospital, $k$: The family is covered for an MRI, $\textbf{l}$: MRI today, $m$: No availability today, $n$: It is an emergency!$\}$. Arguments $a,c,l$ are called the decision variables, their acceptance being the criterion triggering a decision: CT, MRI in two days, or MRI today.

The obtained argumentation system is a possibility of graph that can be extracted from this dialogue.
The extraction process can be done automatically
using so-called argument mining methods~\cite{lippi2016argumentation}.

\begin{figure}[t]
    \centering
    \begin{tikzpicture}[scale=0.8,transform shape, node distance={15mm}, main/.style = {draw, circle}]
            
                \node[main] (1) {$a$}; 
                \node[main] (2) [left of=1] {$b$};
                \node[main] (3) [right of=1] {$c$};
                \node[main] (4) [right of=3] {$l$};
                \node[main] (5) [below right of=3] {$h$};
                \node[main] (6) [left of=5] {$d$};
                \node[main] (7) [right of=5] {$i$};
                \node[main] (8) [above of=3] {$n$};
                \node[main] (9) [above of=4] {$m$};
                \node[main] (10) [below of=6] {$e$};
                \node[main] (11) [below of=5] {$j$};
                \node[main] (12) [below of=7] {$k$};
                \node[main] (13) [left of=10] {$f$};
                \node[main] (14) [left of=13] {$g$};

                \draw[->] (2) -- (1);
                \draw[->] (5) -- (3);
                \draw[->] (6) -- (3);
                \draw[->] (7) -- (3);
                \draw[->] (8) -- (3);
                \draw[->] (5) -- (4);
                \draw[->] (6) -- (4);
                \draw[->] (7) -- (4);
                \draw[->] (9) -- (4);
                \draw[->] (10) -- (6);
                \draw[->] (13) -- (10);
                \draw[->] (14) -- (13);
                \draw[->] (11) -- (5);
                \draw[->] (12) -- (7);
                \draw[->] (8) -- (9);
                
            \end{tikzpicture}
    \caption{
    Argumentation graph associated with Example~\ref{ex:IRM_ou_radio}.}
    \label{fig:ex_radio}
\end{figure}
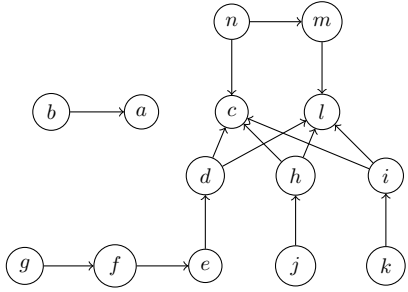
\end{example}

Note that this is a static representation of the dialogue from which any notion of temporality has been erased. Thus, if the arguments had been stated in a different order, it would not change the graph. This will be important to address causality in Section~\ref{sec:discu_causalite}.


Once an argumentation graph has been constructed, it is possible to reason on it to determine sets of arguments that can be considered as accepted.

The \emph{set of direct attackers} of~$x \in A$ is denoted by~$Att_{x}=\{ y\in A \mid (y,x)\in R\}$. A set~$S$ is \emph{conflict-free} if~${\forall (x,y) \in S^2}$, $(x,y)\notin R$. An argument~$x \in A$ is \emph{acceptable} by~$S$ if~$\forall y\in 
Att_x, \exists z \in S \cap Att_y$.

Then an \emph{admissible set}~$S$ is defined as a conflict-free set whose 
elements are all acceptable by~$S$ itself. 
For an acyclic graph, this is the only extension-based semantics as all others coincide
with it and are therefore not discussed here~\cite{rahwan2009argumentation}. 

\addtocounter{example}{-1}
\begin{example} (continued) --
The argument graph
is acyclic. To determine the set of acceptable arguments, it is sufficient to start from the non-attacked arguments, here~${b,g,j,k,n}$. They are accepted by default. Then, an argument attacked by at least one accepted argument cannot be accepted. By applying this rule, we obtain that argument $l$ is accepted, in contrast to $a$ and $c$. Therefore, the final decision is to perform an emergency MRI today.
\end{example}

%% file: 2_Action_Language.tex
\section{Action Language and Causality}
\label{sec:Action_language}

    This section introduces the formal aspects of the action language proposed by~\cite{sarmiento_action_2022-1} and then briefly describes what is considered to be an actual cause in this formalism. For more details refer to~\cite{sarmiento_action_2023}.

    \subsection{Syntax and Semantics}
    \label{sec:Modele}
        
        The purpose of the action language introduced by~\cite{sarmiento_action_2022-1} is to determine the evolution of the world given a set of actions corresponding to deliberate choices of the agent. These actions might trigger some chain reaction through external events. Therefore, in order to have a complete knowledge of the evolution of the world, \cite{sarmiento_action_2022-1} keep track of both: the evolution of the states of the world and the occurrence of events. 
        Hence, we denote by $\mathbb{F}$ the set of 
        variables describing the state of the world, more precisely \emph{ground fluents} representing time-varying properties, and by $\mathbb{E}$ 
        the set of variables describing transitions, more precisely \emph{ground events} that modify fluents.
        
        A \emph{fluent literal} is either a fluent~$f\in \mathbb{F}$, or its negation~$\neg f$. The set of fluent literals in~$\mathbb{F}$ is denoted by~$Lit_{\mathbb{F}}$, i.e.~$Lit_{\mathbb{F}} = \mathbb{F}\cup\left\{\neg f \mid f\in \mathbb{F} \right\}$. The complement of a fluent literal $l$ is defined as~$\overline{l}=\neg f$ if~$l=f$ or~$\overline{l}=f$ if~$l=\neg f$.
        
        \begin{definition}[state]\label{def:state}
            A set $L\subseteq Lit_{\mathbb{F}}$ is a \emph{state} if it is:
                \begin{itemize}
                    \item Coherent: $\forall l\in L, \overline{l}\not\in L$;
                    \item Complete: $\forall f\in \mathbb{F}, f\in L$ or $\neg f \in L$.
                \end{itemize}
        \end{definition}
        
        A state 
        thus gives the value of each of the fluents describing the world. Time is modelled linearly and in a discrete way to associate a state~$S(t)$ to each time point~$t$ of a set~$\mathbb{T} = \left\{-1,0,\dots,N\right\}$. $S(0)$ is the \emph{initial state}. Using a bounded past formalisation, all states before $t=0$ are gathered in a state~$S(-1) =\mathbb{F} \setminus S(0)$.
        
        An event~$e\in\mathbb{E}$ is an atomic formula. Each event is characterised by three elements: preconditions and  triggering conditions give conditions that must be satisfied by a state~$S$  for the event to be triggered (their difference is detailed later in the section); effects indicate the changes to the state that are expected if the event occurs. Note the deliberate use of the term `expected' as an event may have fewer effects than those formalised.
        
        
        The preconditions and effects are represented as formulas of the languages~$\mathcal{P}\Coloneqq l|\psi_1 \wedge \psi_2|\psi_1 \vee \psi_2$ and~$\mathcal{E}\Coloneqq l|\varphi_1 \wedge \varphi_2$, respectively. The functions which associate preconditions, triggering conditions and effects with each event are respectively defined as: $pre: \mathbb{E} \rightarrow \mathcal{P}$, $tri: \mathbb{E} \rightarrow \mathcal{P}$, and~$e\ff: \mathbb{E} \rightarrow \mathcal{E}$. 
        $\mathbb{E}$ is partitioned into two  disjoint sets: 
        $\mathbb{A}$ contains the actions carried out by an agent and thus subjected to their volition; $\mathbb{U}$ contains the exogenous events which are triggered as soon as all the~$pre$ conditions are fulfilled, therefore without the need for an agent to perform them. Thus, for exogenous events~$pre$ and~$tri$ are the same. By contrast, for actions, $tri$ conditions necessarily include $pre$ conditions but those are not sufficient: the~$tri$ conditions of an action also include the volition of the agent or some kind of manipulation by another agent.
        
        The set of all events which occur at time point~$t$ is denoted by~$E(t)$. Allowing concurrency of events (meaning that more than one event can occur at each time point) is one of the main advantages of this action language.
                
        These definitions lead to a classical transition system: $E(t)$ 
        generates the transition between the states~$S(t)$ and~$S(t+1)$. Thus, the states follow one another as events occur, simulating the evolution of the world.
        
        With a bounded past formalisation, events that occurred before~$t=0$ must be represented in order to obtain causal results that are consistent with the philosophical conception of causality.
        Thus, for each fluent literal~$l\in S(0)$ an event~$ini_l\in\mathbb{E}$ is introduced, such that~$e\ff(ini_l)=l$. Then, $E(-1)=\left\{ini_l,l\in S(0)\right\}$ which satisfies~${e\ff(E(-1))=S(0)}$.
        
        To solve potential conflicts or to prioritise between events, a strict partial order~$\succ_{\mathbb{E}}$ is introduced, which ensures the triggering primacy of one event over another.
        
        \begin{definition}[context $\kappa$]\label{def:context}
            The \emph{context}, denoted as $\kappa$, is the octuple $(\mathbb{E},\mathbb{F},pre, tri,e\ff,S(0),\succ_{\mathbb{E}},\mathbb{T})$, where $\mathbb{E}$, $\mathbb{F}$, $pre$, $tri$, $e\ff$, $S(0)$, $\succ_{\mathbb{E}}$, and $\mathbb{T}$ are as defined above.
        \end{definition}
        
        
        \begin{definition}[valid execution]\label{def:semantics}
            An \emph{execution} is a sequence $E(-1),S(0),E(0),\dots,E(N),S(N+1)$. Such an execution is \emph{valid} given~$\kappa$ if~$\forall t\in\mathbb{T}$:
            \begin{enumerate}
                \item $S(t)\subseteq Lit_{\mathbb{F}}$ is a state according to Definition~\ref{def:state}.
                \item $E(t)\subseteq\mathbb{E}$ verifies:
                \begin{enumerate}
                    \item $\forall e\in E(t)$, $S(t) \models pre(e)$;
                    \item $\nexists (e,e')\in E(t)^2,~e\succ_{\mathbb{E}}e'$;
                    \item $\forall e\in \mathbb{E}$ such that $S(t) \models tri(e)$,\\
                    \phantom{$\forall e\in \mathbb{E}$} $e \in E(t)$ or ${\exists e'\in E(t),}$ $e'\succ_{\mathbb{E}}e$;
                \end{enumerate}
                \item $S(t+1)=\left\{l\in S(t),\forall e\in E(t),\overline{l}\not\in e\ff(e)\right\}\cup$ \\ \phantom{$S(t+1)=$} $\left\{l\in Lit_{\mathbb{F}},\exists e\in E(t),l\in e\ff(e)\right\}$.
            \end{enumerate}
        \end{definition}
        
        There is potentially more than one valid execution for a given context~$\kappa$. In fact, there is no specification of when actions are performed in~$\kappa$. Adding a set of timed actions~$\sigma\subseteq\mathbb{A}\times\mathbb{T}$ which models volition of agents as an input, called scenario, leads to a unique valid execution. From this unique execution the event trace and state trace we are interested in, denoted by~$\tau_{\sigma,\kappa}^e$ and~$\tau_{\sigma,\kappa}^s$, respectively, can be extracted.
        
        \begin{definition}[traces $\tau_{\sigma,\kappa}^e$ and $\tau_{\sigma,\kappa}^s$]\label{def:traces}
            Given a scenario~$\sigma$ and a context~$\kappa$, the \emph{event trace}~$\tau_{\sigma,\kappa}^e$  is the sequence of events~$E(-1),E(0),\dots,E(N)$ from the execution which is valid given~$\kappa$, such that: $\forall t\in\mathbb{T}, {\forall e\in E(t)}, {e \in \mathbb{A}} \Leftrightarrow (e,t) \in \sigma$.
            The \emph{state trace}~$\tau_{\sigma,\kappa}^s$ is the sequence of states $S(0),S(1),\dots,S(N+1)$ corresponding to~$\tau_{\sigma,\kappa}^e$.
        \end{definition}

    \subsection{Actual Causality}
    \label{sec:causality}
        
        The actual causation definition proposed by~\cite{sarmiento_action_2022-1} is an action language suitable formalisation of Wright's NESS test. Introduced by~\cite{wright_causation_1985}, this test states that: `A particular condition was a cause of a specific consequence if and only if it was a necessary element of a set of antecedent actual conditions that was sufficient for the occurrence of the consequence.'
        
        A causal relation links a cause to an effect. Since action languages represent the evolution of the world as a succession of states produced by the occurrence of events, states are introduced between events. Therefore, in addition to the actual causality relation that links two occurrences of events, as commonly accepted by philosophers, it is necessary to define causal relations where causes are occurrences of events and effects are formulas of the language~$\mathcal{P}$ that are true at a given time. These intermediate relations are established on the basis of Wright's NESS test of causation. In order to give an actual causality definition suitable for action languages, three causal relations are introduced by~\cite{sarmiento_action_2022-1}: (i)~\emph{Direct NESS-causes} give essential information about causal relations by looking at the effects that the occurrence of an event has actually had, which are not necessarily the same as those expected. Direct NESS-causes relate occurrences of events and formulas of $\mathcal{P}$ being true at a specific time point. However, the set of direct NESS-causes of a formula of $\mathcal{P}$ may include exogenous events that are not necessarily relevant. It is therefore essential to establish a causal chain by going back in time in order to find the set of actions that led to the formula truthfulness. (ii)~\emph{NESS-causes} allow for such a causal chain to be found. If we denote by~$\psi\in\mathcal{P}$ the formula true at $t_\psi$ we are interested in, and~$C$ the set of direct NESS-causes of~$(\psi,t_\psi)$, finding the NESS-causes means finding what causes~$(tri(C),t)$ necessarily, where~$t<t_\psi$. Note that direct NESS-causes are by definition a special case of NESS-causes. (iii)~The occurrence of a first event~$e$ is considered an \emph{actual cause} of the occurrence of a second event~$e'$ if and only if the occurrence of~$e$ is a NESS-cause of the triggering of~$e'$. From this we can deduce that, if the occurrence~$(e',t_2)$ is a direct NESS-cause of~$(\psi,t_3)$ and the occurrence~$(e,t_1)$ is an actual cause of~$(e',t_2)$, with~$t_1<t_2<t_3$, then the occurrence~$(e,t_1)$ is a NESS-cause of~$(\psi,t_3)$. These three causal relations are illustrated using Example~\ref{ex:IRM_ou_radio} in Section~\ref{sec:discu_causalite}.

%% file: 3_Transformation.tex
\section{From AAF to Action Languages}
\label{sec:translation}
   
   This section presents our first contribution: a formalisation of acyclic AAF into the action language introduced above. Section~\ref{sec:contextTranslation} presents the definition of the argumentative context~$\kappa$, Section~\ref{sec:semanticsTranslation} provides the modified definitions of the action language semantics, and Section~\ref{sec:ASP} briefly sketches the structure of the ASP implementation.
    
    In contrast to AAF, we propose to take into account the order of enunciation of arguments. Instead of having only a couple~$(A,R)$, the input is a couple~$(\Delta,R)$, where $\Delta$ is a dialogue, i.e. a sequence of statements in natural language: 
    \begin{definition}[dialogue $\Delta$]\label{def:dialogue}
        A \emph{dialogue} is~${\Delta = \{(a,o) \mid a\in A, o \in \mathbb{N} \}}$, 
        where each argument~$a$ is associated to its order of enunciation,~$o$.
    \end{definition}
        
    \subsection{Instantiating the Context}
    \label{sec:contextTranslation}

        In order to formalise an AAF in the action language described in Section~\ref{sec:Action_language}, let us first define the variables necessary to describe the world, i.e. the AAF. These variables correspond to the fluents $\mathbb{F}$. As introduced in Section~\ref{sec:AAF}, there are two elements to consider: the arguments and the attack relation. First, to describe an argument~$x$, we create two fluents: $p_x \in\mathbb{F}$ and $a_x\in\mathbb{F}$ expressing whether the argument is present in the graph and whether it is acceptable. Regarding~$R$, we use the fluent $cA_{y,x}\in\mathbb{F}$ to model that~$y$ can attack argument~$x$. As we only deal with acyclic AAF, $\nexists (x_1,\dots,x_n)\in A$ such that~$\left(cA_{x_1,x_2},\dots, cA_{x_{n-1},x_n}, cA_{x_n,x_1}\right)\in\mathbb{F}$. We call this property acyclicity of the fluents~$cA$. 
        
        In an AAF, the only deliberate action is to enunciate an argument, which leads to ${\mathbb{A} = \{enunciate_x \mid x \in A \}}$. For this action to be possible, argument $x$ must not have already been said. $x$ then becomes present and acceptable by default. This choice is justified by the fact that its acceptability is evaluated in the next state before it has an impact on the rest of the graph. Formally:
        \begin{align*}
            &pre(enunciate_x)\equiv\neg p_x\\
            &e\ff(enunciate_x)\equiv p_x\wedge a_x
        \end{align*}
        
        \textbf{Remark --} None of the events described below has the effect of making an argument not present. This implies that it is not possible to enunciate an argument that is already present. This assumption does not contradict the framework of classical argumentation. Indeed, a repeated argument would be manifested by an identical but differently named argument in the graph, which is obviously possible with our transformation. However, since the action language we use provides tools for taking temporality into account, there may be a better approach, but it would require further study. Nevertheless, as a first step, this article aims to lay a solid foundation at the cost of some simplifying assumptions.
        
        Before enunciating the next argument, we choose to update the acceptability of all other arguments present after the enunciation of a new argument. This defines a state which we call \emph{argumentative state}.
        \begin{definition}[argumentative state]\label{def:admissible_state}
            A state~$S(t)$ is an \emph{argumentative state} if:\\ 
            i) $\forall x,y, \left[S(t)\models a_x\wedge p_y\wedge cA_{y,x}\Rightarrow S(t)\models\neg a_y\right]$;\\
            ii) $\forall x, \left[S(t)\models p_x \wedge \left(\bigwedge_y \neg a_y\vee\neg cA_{y,x}\right) \Rightarrow S(t)\models a_x\right]$.
        \end{definition}

        After an argument is enunciated, we want updates to be triggered automatically. We represent them with two exogenous events: ${makesUnacc_{y,x}\in\mathbb{U}}$ and ${makesAcc_x\in\mathbb{U}}$. 
        An argument is acceptable only if it is unattacked or attacked only by unacceptable arguments. Hence, it is enough for one of the attackers to be acceptable to make the attacked argument unacceptable. The two cases are considered:
        
        \textit{Acceptability update:} Suppose that an argument $y$ just enunciated can attack argument $x$, and that $x$ and $y$ are acceptable. Then, $x$ being attacked by an acceptable argument~$y$, it becomes unacceptable. Formally, the exogenous event $makesUnacc_{y,x}$ can be written as:
            \begin{align*}
                tri(makesUnacc_{y,x})\equiv &a_x\wedge a_y\wedge cA_{y,x}\\
                e\ff(makesUnacc_{y,x})\equiv &\neg a_x
            \end{align*}
            
        This definition also allows dealing with cases where a new argument~$z$ 
        makes an attacker $y$ of $x$ acceptable again. In this case, $x$ becomes unacceptable.
        
        \textit{Non-acceptability update:} Suppose that argument~$x$ is not acceptable and that an argument~$z$ has just been enunciated. This argument has no direct link with~$x$ but may 
        impact the acceptability of some attackers of~$x$. We therefore check whether all the arguments that 
        can attack $x$ are acceptable or not. If none of them are indeed acceptable, then~$x$ becomes acceptable again. In the action language, this is expressed by the exogenous event $makesAcc_x$:
            \begin{align*}
                tri(makesAcc_x) \equiv &p_x\wedge\neg a_x\wedge \left(\bigwedge_{y} \neg cA_{y,x} \vee \neg a_y \right)\\
                e\ff(makesAcc_x)\equiv &a_x
            \end{align*}
            
        Finally, when an argument~$x$ is enunciated, it must be checked that it has not become unacceptable because of an argument~$y$ already present before it makes other arguments unacceptable. This is reflected in the following priority rule:
            \begin{align*}
                makesUnacc_{y,x} \succ_{\mathbb{E}} makesUnacc_{x,z}
            \end{align*}
            
        Note that adding an argument to the graph can only directly impact the other arguments by making them unacceptable. For this reason, it is not necessary to establish a priority rule of the form~$makesUnacc_{y,x} \succ_{\mathbb{E}} makesAcc_{z}$ as this situation is already addressed by the previous rule.
            
        \textbf{Remark --} In the above transformation, we do not distinguish between the notions of potential and real attack, because such a difference disappears in the equations. Indeed, let us consider a fluent ${att_{y,x}\in\mathbb{F}}$ translating the fact that argument~$y$ actually attacks argument~$x$. Let us define the exogenous event $isAttacking_{y,x}\in\mathbb{U}$ as:
        \begin{align*}
            tri(isAttacking_{y,x})\equiv &p_x\wedge p_y\wedge cA_{y,x}\\
            e\ff(isAttacking_{y,x})\equiv &att_{y,x}
        \end{align*}
        From this definition, an argument~$y$ attacks an argument~$x$ if both are present and~$y$ can attack~$x$. However, for this attack to be taken into account, the attacker~$y$ must be acceptable. We obtain conditions of the form $a_y \wedge att_{y,x}$, i.e. $a_y \wedge p_y\wedge cA_{y,x}$. However, an argument cannot be acceptable without being present, i.e. $a_y \wedge p_y \equiv a_y$. Thus, taking this new fluent into account, we would have: $tri(makesUnacc_{y,x}) \equiv  a_y \wedge a_x \wedge att_{y,x} = a_y \wedge a_x\wedge cA_{y,x}$. The same precondition applies as without the introduction of $isAttacking \text{ and } att$. Therefore, we 
        use only~$cA$.
    
    \subsection{Semantics Adapted to AAF}
    \label{sec:semanticsTranslation}
    
        Having an adapted~$\kappa$ for the argumentative framework, we propose to modify the action language semantics to produce traces that are representative of the reality. For this purpose, arguments will be stated from argumentative states step by step in the order determined by the dialogue~$\Delta$.
        
        The current form of scenario~$\sigma$ is not ideal for this task. Indeed, it implies that we need to know in advance how many steps each chain of admissibility update events will take to plan at which time the next argument should be stated. To solve this issue we introduce a set of ranked actions~$\varsigma\subseteq\mathbb{A}\times\mathbb{N}$ which is called \emph{sequence}. The input to obtain unique traces will no longer be the scenario~$\sigma$ but the sequence~$\varsigma$. These modifications require changes to Definitions~\ref{def:semantics} and~\ref{def:traces}.
        
        \begin{definition}[argumentative setting $\chi$]\label{def:arg_setting}
            The \emph{argumentative setting} of the action language, denoted by~$\chi$, is the couple~$(\varsigma,\kappa)$ with~$\varsigma$ a sequence and~$\kappa$ a context.
        \end{definition}
        
        Definition~\ref{def:semantics_arg} is the result obtained after modifying Definition~\ref{def:semantics}. Conditions 2.d and 2.e are added and~$\forall e\in \mathbb{E}$ is replaced by~$\forall e\in \mathbb{U}$ in condition 2.c. These modifications respectively express that an action in the sequence can be triggered only if no exogenous event is triggered at the same time point, and that event sets in the event trace cannot be empty. Conditions 1, 2.a, 2.b, and 3 remain unchanged. So the triggering of exogenous events remains unchanged.
        
        \begin{definition}[valid execution in an argumentative context]\label{def:semantics_arg}
            Given an argumentative context~$\kappa$, a sequence $E(-1),S(0),E(0),\dots,E(N),S(N+1)$ is a \emph{valid execution} w.r.t. $\kappa$ if, in addition to conditions 1, 2.a, 2.b, and 3 of Definition~\ref{def:semantics}, the following conditions are satisfied~$\forall t\in\mathbb{T}$: 
            \begin{enumerate}
                \item[2] $E(t)\subseteq\mathbb{E}$ satisfies:
                \begin{enumerate}
                    \item[2.c] $\forall e\in \mathbb{U}$ such that $S(t) \models tri(e)$,\\
                    $e \in E(t)$ or ${\exists e'\in E(t), ~e'\succ_{\mathbb{E}}e}$;
                    \item[2.d] If $\exists e\in E(t)\cap\mathbb{A}$, then $\forall e'\in\mathbb{U}$, $S(t)\not\models tri(e')$; 
                    \item[2.e] $E(t)\not=\varnothing$.
                \end{enumerate}
            \end{enumerate}
        \end{definition}
        
        In Definition~\ref{def:traces} traces were defined as extracts of a valid execution given~$\kappa$ and additional conditions related to~$\sigma$. Instead of defining directly traces, Definition~\ref{def:traces_arg} corresponds to a valid execution given~$\chi=(\varsigma,\kappa)$. Traces are simply extracts from such valid executions. 
        
        \begin{definition}[valid execution given~$\chi$]\label{def:traces_arg}
            Given an argumentative setting $\chi=(\varsigma,\kappa)$, a valid execution w.r.t.~$\kappa$, is \emph{valid w.r.t. to~$\chi$} if:
            \begin{enumerate}
                \item ${\forall t\in\mathbb{T}}$, ${E(t)\subset\left(\left\{a,\exists o\in\mathbb{N}, (a,o)\in\varsigma\right\}\cup\mathbb{U}\right)}$;
                \item ${\forall \left(\left(e,o\right),\left(e',o'\right)\right)\in\varsigma^2}$ such that $o<o'$,\\
                $\exists t,t'$ such that $e\in E(t) \mbox{ and }  e'\in E(t') \mbox{ and } t<t'$;
                \item ${\forall \left(\left(e,o\right),\left(e',o'\right)\right)\in\varsigma^2}$ such that $o=o'$,\\
                $\exists t$ such that $(e,e')\in E(t)^2$.
            \end{enumerate}
        \end{definition}
    
        Given a valid execution given~$\chi$, its \emph{event trace}~$\tau_{\chi}^e$ is its sequence of events $E(-1),E(0),\dots,E(N)$, its \emph{state trace}~$\tau_{\chi}^s$ is its sequence of states $S(0),S(1),\dots,S(N+1)$.
        
    \subsection{ASP Implementation}
    \label{sec:ASP}
        
        We propose an adapted implementation in ASP based on the sound and complete one described in~\cite{sarmiento_action_2023}. The ASP program~$\pi_{con}(\kappa)$ and~$\pi_{seq}(\varsigma)$ are obtained by the translation of the context~$\kappa$ and the sequence $\varsigma$ respectively. $\pi_{\mathbb{A}}$ is obtained by the translation of the action language semantics introduced in Section~\ref{sec:Modele} and modified in Section~\ref{sec:semanticsTranslation}.  $\pi_{\mathbb{C}}$ is obtained by the translation of the causal relations definitions introduced by~\cite{sarmiento_action_2022-1}. 
        The entire program $\Pi(\chi)  = \pi_{sce}(\varsigma)\cup\pi_{con}(\kappa)\cup\pi_{\mathbb{A}}\cup\pi_{\mathbb{C}}$ is  available\footnote{\url{https://gitlab.lip6.fr/sarmiento/kr\_2023.git}}.
        

\section{Formal Properties}
    \label{sec:formalProp}

    This section establishes  formal properties of the proposed transformation. First, we prove that a notion of temporality is captured by the transformation. Then we prove its soundness and completeness. Finally, we introduce the propositions that pave the way to the discussion of Section~\ref{sec:discussion}.
    
    \subsection{Preliminary Property on a Valid Execution}
        
        We start by showing that, although valid executions given~$\kappa$ are not unique, valid executions given~$\chi$ are, and thus are the corresponding traces~$\tau_{\chi}^e$ and~$\tau_{\chi}^s$. 
        
        \begin{proposition}\label{prop:unicity_valid_execution}
            Given an argumentative setting $\chi=(\varsigma,\kappa)$, the traces $\tau_{\chi}^e$ and~$\tau_{\chi}^s$ are unique.
        \end{proposition}
        
        \begin{proof}
            Let us prove by contradiction the unicity of valid executions given~$\chi$. Let~$\chi=(\varsigma,\kappa)$ be the argumentative setting and $\epsilon$, $\epsilon'$ two valid executions given~$\chi$. By way of a reductio ad absurdum, we suppose that~$\epsilon\neq\epsilon'$.
            
            According to Definition~\ref{def:semantics}, $S(t+1)$ is derived from~$S(t)$ and the events in~$E(t)$, and similarly for $S'$. Hence, given that~$\kappa$ is common to~$\epsilon$ and~$\epsilon'$, $E(-1)=E'(-1)$ and $S(0)=S'(0)$, the first discrepancy between $\epsilon$, and $\epsilon'$ is not to be found in a set of states, but in a set of events, which is not empty as the executions are valid.
            Let $t_0$ be the minimal date at which a difference between~$\epsilon$ and~$\epsilon'$ is observed. We have~$E(t_0)\neq E'(t_0)$, $\forall t<t_0$, $E(t)=E'(t)$, and~$\forall t\leq t_0$, $S(t)=S'(t)$. Thus, $\forall e\in\mathbb{E}$, $S(t_0)\models pre(e)\Leftrightarrow S'(t_0)\models pre(e)$. Without loss of generality, let us consider an event~$e_0$ such that $e_0\not\in E(t_0)$ and~$e_0\in E'(t_0)$. Two cases can occur, $e_0\in\mathbb{U}$ or~$e_0\in\mathbb{A}$.
            
            i) Let us first show by contradiction that~$e_0\not\in\mathbb{U}$. Let us suppose~$e_0\in\mathbb{U}$. As~$e_0\in E'(t_0)$ and~$S(t_0)=S'(t_0)$, $S(t_0)\models tri(e_0)$. Then from 2.c in Definition~\ref{def:semantics_arg}, $e_0\not\in E(t_0)$ implies $\exists e\in E(t_0)$ such that~${e\succ_\mathbb{E}e_0}$. Then from 2.b applied to~$E'(t_0)$, we get ${e\not\in E'(t_0)}$. Now, either $e\in\mathbb{A}$ or $e\in\mathbb{U}$. If~$e\in\mathbb{A}$, as $e\in E(t_0)$, condition 2.d would imply that~$S'(t_0)\not\models tri(e_0)$ which contradicts our assumption. In the second case, if~$e\in\mathbb{U}$, then~$e\in E(t_0)$ and~$e\not\in E'(t_0)$ because of the same reasons behind~${e_0\not\in E(t_0)}$ and~${e_0\in E'(t_0)}$. If we apply the same reasoning to~$e$, we get~$\exists e'\in\mathbb{U}$ such that $e'\in E(t_0)$, $e'\not\in E'(t_0)$, and~$e'\succ_\mathbb{E}e$. This reasoning can be repeated again on~$e'$ and so on. As $\mathbb{U}$ is finite, either the chain will be broken, or an event will be used a second time. The first case means~$\exists \tilde{e}\in\mathbb{U}$ such that $\tilde{e}\not\in E(t_0)$ and~$\tilde{e}\in E'(t_0)$ is false, which makes all the chain false. In the second case, by transitivity we get~$\tilde{e}\succ_\mathbb{E}\tilde{e}$. This leads to a contradiction as~$\succ_\mathbb{E}$ is a strict partial order. Thus, $e_0\in\mathbb{U}$ is not possible. 
            
            ii) As from (i)~$e_0\in\mathbb{A}$, by condition 1. in Definition~\ref{def:traces_arg}, ${e_0\in E'(t_0)}$ implies ${e_0\in\left\{a,\exists o\in\mathbb{N}, (a,o)\in\varsigma\right\}}$. $\varsigma$ being the same for~$\epsilon$ and~$\epsilon'$, the rank~$o_0\in\mathbb{N}$ associated to~$e_0$ is the same. Hence, given that~$\forall t<t_0$, $\epsilon(t)=\epsilon'(t)$ and~$e_0\in E'(t_0)$, $e_0\not\in E'(t)$ implies $e_0\not\in E(t)$.
            The only possibility left is the procrastination of actions. In the case where~$\nexists (e,o_0)\in\varsigma$ such that~$e\in E(t_0)$, $E(t_0)=\varnothing$ which is in contradiction with condition 2.e in Definition~\ref{def:semantics_arg}. Otherwise, we have a contradiction with condition 3. in Definition~\ref{def:traces_arg}. Thus, $e_0\in\mathbb{A}$ is not possible.\\
            Since the assumption of the existence of~$t_0$ leads to a contradiction for all above cases, we have no alternative but to reject the existence of a minimal daten$t_0$ at which a difference between~$\epsilon$ and~$\epsilon'$ is observed, and thus to reject that $\epsilon\neq\epsilon'$. The unicity of traces is accordingly established.
        \end{proof}
    
        From now on, when reference is made to events and states, they will be those from the unique~$\tau_{\chi}^e$ and~$\tau_{\chi}^s$, respectively. Thus, the set of all events which actually occurred at time point~$t$ is~$E^{\chi}(t)=\tau_{\chi}^e(t)$. Following the same reasoning, the actual state at time point~$t$ is~$S^{\chi}(t)=\tau_{\chi}^s(t)$.
        
    \subsection{Soundness and Completeness}
        
        In this section, we establish the soundness and completeness of our transformation. For that, we first introduce the notion of associated graph as follows:
        
        \begin{definition}\label{def:associated_graph}
            Given a state~$S^\chi(t)$, $AF' = (A',R')$, where $A' = \{ x \mid S^\chi(t)\models p_x\}$ and $R' = \{ (y,x) \mid S^\chi(t)\models cA_{y,x}\}$, is called the \emph{associated graph} of~$S^\chi(t)$.
        \end{definition}
        
        From the acyclicity property of the fluent~$cA$, the associated graph is acyclic.
        
        Now, we focus on the notion of acceptability. We first characterise argumentative states using $tri$.
        
        \begin{lemma}\label{lem:equiv_tri_arg_state}
            Let~$S^\chi(t)$ be a state. The two following propositions are equivalent:
            
            $\bullet \; \forall e\in\mathbb{U}$, $S^\chi(t)\not\models tri(e)$
            
            $\bullet \; S^\chi(t)$ is an argumentative state as defined in Def.~\ref{def:admissible_state}.
        \end{lemma}
         
        \begin{proof}
            In our context, $\mathbb{U} = \{makesAcc,makesUnacc\}$. We  prove that  $S^\chi(t)\not\models tri(makesAcc_{x})$ is equivalent to (ii) of Definition~\ref{def:admissible_state} and $S^\chi(t)\not\models tri(makesUnacc_{y,x})$ is equivalent to (i) of Definition~\ref{def:admissible_state}: for any~$x, y$
            
            \noindent $ \bullet$ $\neg tri(makesAcc_{x}) =\neg (p_x \wedge \neg a_x \wedge (\bigwedge_y \neg a_y\vee\neg cA_{y,x}))$\\ $= \neg (p_x \wedge (\bigwedge_y \neg a_y\vee\neg cA_{y,x})) \vee a_x$\\ $= p_x \wedge (\bigwedge_y \neg a_y\vee\neg cA_{y,x}) \Rightarrow a_x$, which leads to the desired equivalence with (ii) in Definition~\ref{def:admissible_state}.
            
             \noindent $ \bullet$$ \neg tri(makesUnacc_{y,x}) =  \neg (a_x \wedge a_y \wedge cA_{y,x})$ \\$=  \neg (a_x \wedge p_y \wedge a_y \wedge cA_{y,x}) \mbox{ as $a_y$ implies $p_y$}\\
             =  \neg (a_x \wedge p_y \wedge cA_{y,x}) \vee \neg a_y$\\$=  a_x \wedge p_y \wedge cA_{y,x} \Rightarrow \neg a_y $, which leads to the desired equivalence with (i) in Definition~\ref{def:admissible_state}.
        \end{proof}
        
        An argumentative state can therefore be seen as a state where nothing happens until a voluntary action is made. Now we prove that it is always possible to reach such a state from an argumentative state in which an $x$ is enunciated.
        
        \begin{proposition}\label{prop:arg_state_necess}
            Given an argumentative state~$S^\chi(t)$ and ${x\in A}$, if $enunciate_x\in E^\chi(t)$, then~$\exists t'\in\mathbb{T}$, $t<t'$ such that~$S^\chi(t')$ is an argumentative state.
        \end{proposition}
        
        \begin{proof}
            Given an argumentative state~$S^\chi(t)$  and $x \in A$ such that $enunciate_x \in E(t)$, let us prove that $\exists t'\in\mathbb{T}$, ${t < t'}$ such that no trigger is a logical consequence of $S^{\chi}(t')$, which leads to the desired result using Lemma~\ref{lem:equiv_tri_arg_state}. 
            
            As $\mathbb{U}$ and $\{S^{\chi}(t)\models cA_{y,x}\mid (x,y) \in A^2\}$ are finite sets, there is a finite number of possible triggering for $E^{\chi}(t)$. Moreover, as there is a finite number of arguments, there is a finite number of paths in the associated graph. The graph being acyclic, each path length is finite. Therefore, there is a finite maximum number of sets of events~($M$) and so $\exists t' \leq (t + M + 1)$ such that~${\forall e \in \mathbb{U}, S^{\chi}(t') \nvDash tri(e)}$. 
        \end{proof}
        
        Finally, this proposition allows us to prove that an acceptable argument in the argumentative state is acceptable in the associated graph and vice-versa, as the triggering rules have been made in order to model how acceptability is computed.
        We first prove a useful lemma.
        
        \begin{lemma}\label{lem:equiv_cond_ii}
            Given an argumentative state~$S^\chi(t)$, for any~$x$,\\
            $S^\chi(t) \models p_x \wedge \neg a_x \Leftrightarrow S^\chi(t) \models \exists y, p_x \wedge a_y \wedge cA_{y,x}$.
        \end{lemma}
        
        \begin{proof}
            $[\Rightarrow]:$ For any~$x$ such that $S^\chi(t) \models p_x \wedge \neg a_x$, (ii) of Definition~\ref{def:admissible_state} implies that $S^\chi (t) \models \neg p_x \vee (\bigvee_y a_y \wedge cA_{y,x})$. Therefore, as $S^\chi(t) \models p_x, S^\chi(t) \models \exists y, p_x \wedge a_y \wedge cA_{y,x}$.
            
            $[\Leftarrow]:$ 
            Let us prove it by contraposition: let  $x_0$ be such that $S^\chi(t) \models \neg p_{x_0} \vee a_{x_0}$. If $S^\chi(t) \models \neg p_{x_0}$, then $x_0$ is such that $S^\chi(t) \models \forall y, \neg p_{x_0} \vee \neg a_y \vee \neg cA_{y,x_0}$, which ends the proof. Otherwise, $S^\chi(t) \models p_{x_0} \wedge a_{x_0}$. If $S^\chi(t) \models p_{x_0}\wedge (\exists y, a_y \wedge cA_{y,x_0})$ then 
            $S^\chi (t) \models tri(makesUnacc_{y,x_0})$, which is not possible as $S^\chi(t)$ is argumentative. Thus both cases lead to a contradiction.
        \end{proof}
        
        The next proposition establishes the correspondence between acceptability in argumentation and argumentative states.
        
        \begin{proposition}\label{pro:associated_graph}
            Given an argumentative state~$S^\chi(t)$ and its associated graph~${AF=(A,R)}$ according to Definition~\ref{def:associated_graph}, then for any~$x$, $x\in A$ acceptable by $A \iff S^\chi(t)\models a_x$.
        \end{proposition}
        
        \begin{proof}
            $[\Rightarrow] :$ Let $x_0 \in A$ such that $x_0$ is acceptable by A, let us prove that~$S^\chi(t)\models a_{x_0}$.
            
            Let us suppose that~$S^\chi(t)\models \neg a_{x_0}$. Moreover, by construction of AF $S^{\chi}(t)\models p_{x_0}$. $S^{\chi}(t)$ is an argumentative state so according to Lemma~\ref{lem:equiv_cond_ii} $S^{\chi}(t)\models p_{x_0}\wedge\neg a_{x_0} \Leftrightarrow S^{\chi}(t) \models \exists y, p_{x_0} \wedge a_y \wedge cA_{y,{x_0}}$. (i) of Definition~\ref{def:admissible_state} applied to~$a_y$ says that $\forall z, S^{\chi}(t) \models a_y \wedge p_z \wedge cA_{z,y} \Rightarrow S^{\chi}(t)\models \neg a_z$. As there is a finite number of arguments, it is possible to repeat the process we applied for ${x_0}$ on $z$ until one of the two scenarios:
            
            $\bullet$ $S^{\chi}(t) \models \nexists y , p_{z} \wedge a_y \wedge cA_{y,z}$. This leads to trigger the exogenous event $makesAcc_{z}$ which is not possible as $S^{\chi}(t)$ is an argumentative state according to Lemma~\ref{lem:equiv_tri_arg_state}.
            
            $\bullet$ $\forall z,S^{\chi}(t) \models a_y \wedge p_z \wedge cA_{z,y} \Rightarrow S^{\chi}(t)\models\neg a_z$ where $p_z \wedge cA_{z,y}$ is false. Then in $AF$, $Att_y = \emptyset$. Therefore, $y$ is acceptable which contradicts that ${x_0}$ is acceptable.
            
            So, $S^\chi(t)\models a_{x_0}$.
            
           $[\Leftarrow]:$ Let ${x_0} \in A$ such that $S^\chi(t)\models a_{x_0}$. Let us prove that ${x_0}$ is acceptable by $A$.
            
            As $S^\chi(t)\models a_{x_0}$ and $ S^{\chi}(t)$ is argumentative, we have that $ \forall y, S^{\chi}(t) \models  a_{x_0} \wedge p_y \wedge cA_{y,{x_0}} \Rightarrow S^\chi(t)\models \neg a_y$. Then, for any~$y$ satisfying the premise, by definition of $AF$, $({x_0},y) \in A^2$ and $(y,{x_0}) \in R$.
            
            If such a~$y$ is acceptable by $A$, then according to $[\Rightarrow]$, $S^\chi(t)\models a_y$. In that case, $ S^{\chi}(t) \models tri(makesUnacc_{y,{x_0}})$ which contradicts the fact that $ S^{\chi}(t)$ is an argumentative state.
            
            So, as $\forall y\in Att_x$, $y$ is not acceptable by $A$, ${x_0}$ is acceptable by~$A$.
        \end{proof}
        
        We proved that there is an equivalence between an argumentative state and its associated graph. Now, from a dialogue and the attack relation, the traces are generated as well as an AAF. From that point, we establish the existence of a state whose associated graph is equal to the initial AAF. Such a state is called the \emph{final argumentative state} and is defined as an argumentative state~$S^{\chi}(t)$ such that~$\forall x\in A$, $\exists t'\in\mathbb{T}$ such that $t'<t$ and~$enunciate_x\in E^\chi(t')$.
        
        \begin{theorem}[Soundness and Completeness]\label{th:sound_complete}
            Given a dialogue~$\Delta$ and a set of attack~$R$, given the argumentative setting~$\chi$, the associated argumentative graph~$AF'$ of the final argumentative state~$S^{\chi}(t)$, and~${AF=(A,R)}$ obtained from~$(\Delta,R)$, it holds that $AF'=AF$.
        \end{theorem}
        
        \begin{proof}
            As $AF'$ is associated to a final argumentative state,~$\forall x\in A$, $\exists t'\in\mathbb{T}$ such that $t'<t$ and~${enunciate_x\in E^\chi(t')}$. Now, $eff(enunciate_x) = p_x \wedge a_x$. So, $A' = A$.
            
            Moreover, by construction of $cA_{y,x}$ and $R'$, $R=R'$. So $AF = AF'$.
            
            Finally, from Proposition~\ref{pro:associated_graph} as $S^{\chi}(t)$ is argumentative, $\forall x \in A=A', S^{\chi}(t)\models a_x \Leftrightarrow x$ is acceptable by $A$.
        \end{proof}
        
    \subsection{On Temporality and Causality}
    
            The preliminary Proposition~\ref{prop:unicity_valid_execution} highlights the fact that temporality is captured by the proposed transformation. Indeed, given an order of enunciation, as expressed by sequence~$\varsigma$, 
        there exists a unique trace of states corresponding to a unique way to traverse a graph. When only given a context~$\kappa$, this unicity property does not hold. This section  
        shows that this temporality does not impact the final argumentative state  but impacts the causal relations.
        
        \begin{proposition}\label{prop:indep_order}
            Let~$\varsigma$ and~$\varsigma'$ be sequences such that~$\varsigma'$ is a permutation of the ranks of~$\varsigma$. Given the final argumentative states~$S^{\varsigma,\kappa}(t)$, $S^{\varsigma',\kappa}(t')$, belonging to~$\tau_{\varsigma,\kappa}^s$ and~$\tau_{\varsigma',\kappa}^s$, with $(t,t')\in\mathbb{T}\times\mathbb{T'}$, $S^{\varsigma,\kappa}(t)=S^{\varsigma',\kappa}(t')$. 
        \end{proposition}
        
        \begin{proof}
            Let us call $AF$ and $AF'$, the associated graphs of  the final argumentative states $S^{\varsigma,\kappa}(t)$ and $S^{\varsigma',\kappa}(t')$, respectively. Given that they have the same actions in the sequence, then $A = A'$. They also share the same context so $R = R'$. Therefore $AF = AF'$.
            
            Now, according to Proposition~\ref{pro:associated_graph}, $x \in A$ acceptable by $A \Leftrightarrow S^{\varsigma,\kappa}(t)\models a_x$. So $\forall x, S^{\varsigma,\kappa}(t)\models a_x \Leftrightarrow \forall x, S^{\varsigma',\kappa}(t')\models a_x$.
        \end{proof}
        
        This property implies that the final argumentative state does not depend on~$\varsigma$, but only on the set of arguments it contains: no matter the order in which arguments are enunciated, the final argumentative state is always the same. This immediately leads to the following unicity corollary:  
        
        \begin{corollary}\label{cor:unique_arg_state}
            Given an~${AF=(A,R)}$, $\exists! S^\chi(t)$ final argumentative state which associated argumentative graph is~${AF=(A,R)}$.
        \end{corollary}
        
        
        Proposition~\ref{prop:indep_order} and its corollary are in accordance with AAF. The relevance of temporality integration comes from the intermediate states, as illustrated in the next section, and from the causal relations that can be derived from it:
        
        \begin{proposition}\label{prop:dep_ordre_causes}
            Causal relations depend on the sequence~$\varsigma$.
        \end{proposition}
        
        \begin{proof}
        This proposition is proved by example, commented in details in the next section that illustrated the effect of considering Example~\ref{ex:IRM_ou_radio}, and a modification thereof in Example~\ref{ex:IRM_ou_radio_modif}.
            Let $\Pi(\chi)$ be the program obtained given~$\kappa,\varsigma$ of Example~\ref{ex:IRM_ou_radio}, as described in Section~\ref{sec:ASP}, and let $\Pi(\chi')$ be the program similarly obtained given~$\kappa,\varsigma'$ of Example~\ref{ex:IRM_ou_radio_modif}. Given the NESS-cause definition in~\cite{sarmiento_action_2023}, ${\Pi(\chi)\models ness(o(enunciate_d,4),h(neg(a_c),31))}$, where occurrence of events $(e,t)\in\mathbb{E}\times\mathbb{T}$ are represented by the predicate $o(e,t)$ and the truthfulness of $\mathcal{P}$ formulas $(\psi,t)\in\mathbb{F}\times\mathbb{T}$ by the predicate $h(\psi,t)$, but $\nexists t,t'\in\mathbb{T}^2$, ${\Pi(\chi')\models ness(o(enunciate_d,t),h(neg(a_c),t'))}$.
        \end{proof}
        
        This proposition shows that causal relations depend on the order in which arguments are enunciated. Thus, even if, as in the classical framework of argumentation, the acceptability of an argument in the final argumentative state does not depend on it (cf~Theorem~\ref{th:sound_complete}), it is still essential to take temporality into account when dealing with notions close to causality, especially in explicability.

%% file: 4_Discussion.tex
\section{Application to the Example and Discussion}
\label{sec:discussion}

    This section proposes to illustrate the proposed transformation from AAF to action language for Example~\ref{ex:IRM_ou_radio}, highlighting its exploitation to get enriched information about the modelled dialogue, more precisely for providing visual representations justifying the acceptance or rejection of arguments. It considers successively two classes of argumentation explanations in \cite{cyras_survey}'s taxonomy: 
    it first shows how it can lead to graphical representations of the processes of accepting/rejecting arguments, it then discusses the case of causal explanations.

        
    \subsection{Graphical Representation and Explanation}
    \label{sec:discu_temps}

    According to \cite{cyras_survey}, argumentation explanations can consist in extracting argumentative subgraphs to justify the acceptance or rejection of an argument for a given AAF semantics, producing a graphical representation of the underlying process.

        
The transformation proposed in the previous section makes it possible to derive  graphical representations of the argumentative process. Indeed, the traces of events and states can be used  to obtain a narrative of the interaction that can be represented graphically. The visualisation we propose is illustrated in Figure~\ref{fig:exemple_ini}, in a simplified form, for Example~\ref{ex:IRM_ou_radio}. It is enriched in the next section using causality relations.

        \begin{figure}[t]
            \centering
        	\resizebox{7cm}{!}{\input{fig_exemple_ini}}
        	\caption{Partial graphical representation 
        	associated to Example~\ref{ex:IRM_ou_radio}. 
        	Hexagons represent fluents and rectangles events.
        	}
        	\label{fig:exemple_ini} 
        \end{figure}
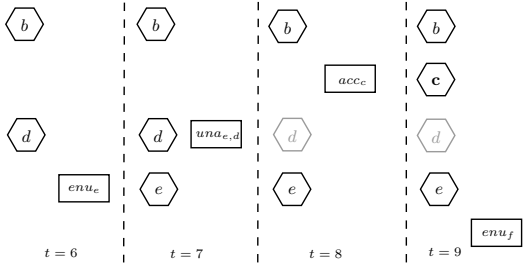

        \begin{table}[t]
        \footnotesize
        \addtolength{\tabcolsep}{-1pt}
        \begin{center}
            \begin{tabular}{|c|cccccccccccc|c|}
        \hline
         & $a$ & \multicolumn{1}{|c|}{$b$} & \multicolumn{1}{c|}{$c$} & \multicolumn{1}{c|}{$d$} & \multicolumn{1}{c|}{$e$} & \multicolumn{1}{c|}{$f$} & \multicolumn{1}{c|}{$g$} & \multicolumn{1}{c|}{$h$,$i$} & \multicolumn{1}{c|}{$j$} & \multicolumn{1}{c|}{$k$} & \multicolumn{1}{c|}{$l$} & $m$ & $n$ \\ \hline
        $a$ & \multicolumn{1}{c|}{\scriptsize{$\bullet$}} & \multicolumn{1}{c|}{\scriptsize{$\circ$}} & \multicolumn{1}{c|}{\scriptsize{$\circ$}} & \multicolumn{1}{c|}{\scriptsize{$\circ$}} & \multicolumn{1}{c|}{\scriptsize{$\circ$}} & \multicolumn{1}{c|}{\scriptsize{$\circ$}} & \multicolumn{1}{c|}{\scriptsize{$\circ$}} & \multicolumn{1}{c|}{\scriptsize{$\circ$}} & \multicolumn{1}{c|}{\scriptsize{$\circ$}} & \multicolumn{1}{c|}{\scriptsize{$\circ$}} & \multicolumn{1}{c|}{\scriptsize{$\circ$}} & \scriptsize{$\circ$} & \scriptsize{$\circ$} \\ \hline
        $b$ & \multicolumn{1}{c|}{\cellcolor[HTML]{B2BEB5}} & \multicolumn{1}{c|}{\scriptsize{$\bullet$}} & \multicolumn{1}{c|}{\scriptsize{$\bullet$}} & \multicolumn{1}{c|}{\scriptsize{$\bullet$}} & \multicolumn{1}{c|}{\scriptsize{$\bullet$}} & \multicolumn{1}{c|}{\scriptsize{$\bullet$}} & \multicolumn{1}{c|}{\scriptsize{$\bullet$}} & \multicolumn{1}{c|}{\scriptsize{$\bullet$}} & \multicolumn{1}{c|}{\scriptsize{$\bullet$}} & \multicolumn{1}{c|}{\scriptsize{$\bullet$}} & \multicolumn{1}{c|}{\scriptsize{$\bullet$}} & \scriptsize{$\bullet$} & \scriptsize{$\bullet$} \\ \cline{1-1} \cline{3-14} 
        $c$ & \cellcolor[HTML]{B2BEB5} & \multicolumn{1}{c|}{\cellcolor[HTML]{B2BEB5}} & \multicolumn{1}{c|}{\scriptsize{$\bullet$}} & \multicolumn{1}{c|}{\scriptsize{$\circ$}} & \multicolumn{1}{c|}{\scriptsize{$\bullet$}} & \multicolumn{1}{c|}{\scriptsize{$\circ$}} & \multicolumn{1}{c|}{\scriptsize{$\bullet$}} & \multicolumn{1}{c|}{\scriptsize{$\circ$}} & \multicolumn{1}{c|}{\scriptsize{$\circ$}} & \multicolumn{1}{c|}{\scriptsize{$\bullet$}} & \multicolumn{1}{c|}{\scriptsize{$\bullet$}} & \scriptsize{$\bullet$} & \scriptsize{$\circ$} \\ \cline{1-1} \cline{4-14} 
        $d$ & \cellcolor[HTML]{B2BEB5} & \cellcolor[HTML]{B2BEB5} & \multicolumn{1}{c|}{\cellcolor[HTML]{B2BEB5}} & \multicolumn{1}{c|}{\scriptsize{$\bullet$}} & \multicolumn{1}{c|}{\scriptsize{$\circ$}} & \multicolumn{1}{c|}{\scriptsize{$\bullet$}} & \multicolumn{1}{c|}{\scriptsize{$\circ$}} & \multicolumn{1}{c|}{\scriptsize{$\circ$}} & \multicolumn{1}{c|}{\scriptsize{$\circ$}} & \multicolumn{1}{c|}{\scriptsize{$\circ$}} & \multicolumn{1}{c|}{\scriptsize{$\circ$}} & \scriptsize{$\circ$} & \scriptsize{$\circ$} \\ \cline{1-1} \cline{5-14} 
        $e$ & \cellcolor[HTML]{B2BEB5} & \cellcolor[HTML]{B2BEB5} & \cellcolor[HTML]{B2BEB5} & \multicolumn{1}{c|}{\cellcolor[HTML]{B2BEB5}} & \multicolumn{1}{c|}{\scriptsize{$\bullet$}} & \multicolumn{1}{c|}{\scriptsize{$\circ$}} & \multicolumn{1}{c|}{\scriptsize{$\bullet$}} & \multicolumn{1}{c|}{\scriptsize{$\bullet$}} & \multicolumn{1}{c|}{\scriptsize{$\bullet$}} & \multicolumn{1}{c|}{\scriptsize{$\bullet$}} & \multicolumn{1}{c|}{\scriptsize{$\bullet$}} & \scriptsize{$\bullet$} & \scriptsize{$\bullet$} \\ \cline{1-1} \cline{6-14} 
        $f$ & \cellcolor[HTML]{B2BEB5} & \cellcolor[HTML]{B2BEB5} & \cellcolor[HTML]{B2BEB5} & \cellcolor[HTML]{B2BEB5} & \multicolumn{1}{c|}{\cellcolor[HTML]{B2BEB5}} & \multicolumn{1}{c|}{\scriptsize{$\bullet$}} & \multicolumn{1}{c|}{\scriptsize{$\circ$}} & \multicolumn{1}{c|}{\scriptsize{$\circ$}} & \multicolumn{1}{c|}{\scriptsize{$\circ$}} & \multicolumn{1}{c|}{\scriptsize{$\circ$}} & \multicolumn{1}{c|}{\scriptsize{$\circ$}} & \scriptsize{$\circ$} & \scriptsize{$\circ$} \\ \cline{1-1} \cline{7-14} 
        $g$ & \cellcolor[HTML]{B2BEB5} & \cellcolor[HTML]{B2BEB5} & \cellcolor[HTML]{B2BEB5} & \cellcolor[HTML]{B2BEB5} & \cellcolor[HTML]{B2BEB5} & \multicolumn{1}{c|}{\cellcolor[HTML]{B2BEB5}} & \multicolumn{1}{c|}{\scriptsize{$\bullet$}} & \multicolumn{1}{c|}{\scriptsize{$\bullet$}} & \multicolumn{1}{c|}{\scriptsize{$\bullet$}} & \multicolumn{1}{c|}{\scriptsize{$\bullet$}} & \multicolumn{1}{c|}{\scriptsize{$\bullet$}} & \scriptsize{$\bullet$} & \scriptsize{$\bullet$} \\ \cline{1-1} \cline{8-14} 
        $h$ & \cellcolor[HTML]{B2BEB5} & \cellcolor[HTML]{B2BEB5} & \cellcolor[HTML]{B2BEB5} & \cellcolor[HTML]{B2BEB5} & \cellcolor[HTML]{B2BEB5} & \cellcolor[HTML]{B2BEB5} & \multicolumn{1}{c|}{\cellcolor[HTML]{B2BEB5}} & \multicolumn{1}{c|}{\scriptsize{$\bullet$}} & \multicolumn{1}{c|}{\scriptsize{$\circ$}} & \multicolumn{1}{c|}{\scriptsize{$\circ$}} & \multicolumn{1}{c|}{\scriptsize{$\circ$}} & \scriptsize{$\circ$} & \scriptsize{$\circ$} \\ \cline{1-1} \cline{9-14} 
        $i$ & \cellcolor[HTML]{B2BEB5} & \cellcolor[HTML]{B2BEB5} & \cellcolor[HTML]{B2BEB5} & \cellcolor[HTML]{B2BEB5} & \cellcolor[HTML]{B2BEB5} & \cellcolor[HTML]{B2BEB5} & \multicolumn{1}{c|}{\cellcolor[HTML]{B2BEB5}} & \multicolumn{1}{c|}{\scriptsize{$\bullet$}} & \multicolumn{1}{c|}{\scriptsize{$\bullet$}} & \multicolumn{1}{c|}{\scriptsize{$\circ$}} & \multicolumn{1}{c|}{\scriptsize{$\circ$}} & \scriptsize{$\circ$} & \scriptsize{$\circ$} \\ \cline{1-1} \cline{9-14} 
        $j$ & \cellcolor[HTML]{B2BEB5} & \cellcolor[HTML]{B2BEB5} & \cellcolor[HTML]{B2BEB5} & \cellcolor[HTML]{B2BEB5} & \cellcolor[HTML]{B2BEB5} & \cellcolor[HTML]{B2BEB5} & \cellcolor[HTML]{B2BEB5} & \multicolumn{1}{c|}{\cellcolor[HTML]{B2BEB5}} & \multicolumn{1}{c|}{\scriptsize{$\bullet$}} & \multicolumn{1}{c|}{\scriptsize{$\bullet$}} & \multicolumn{1}{c|}{\scriptsize{$\bullet$}} & \scriptsize{$\bullet$} & \scriptsize{$\bullet$} \\ \cline{1-1} \cline{10-14} 
        $k$ & \cellcolor[HTML]{B2BEB5} & \cellcolor[HTML]{B2BEB5} & \cellcolor[HTML]{B2BEB5} & \cellcolor[HTML]{B2BEB5} & \cellcolor[HTML]{B2BEB5} & \cellcolor[HTML]{B2BEB5} & \cellcolor[HTML]{B2BEB5} & \cellcolor[HTML]{B2BEB5} & \multicolumn{1}{c|}{\cellcolor[HTML]{B2BEB5}} & \multicolumn{1}{c|}{\scriptsize{$\bullet$}} & \multicolumn{1}{c|}{\scriptsize{$\bullet$}} & \scriptsize{$\bullet$} & \scriptsize{$\bullet$} \\ \cline{1-1} \cline{11-14} 
        $l$ & \cellcolor[HTML]{B2BEB5} & \cellcolor[HTML]{B2BEB5} & \cellcolor[HTML]{B2BEB5} & \cellcolor[HTML]{B2BEB5} & \cellcolor[HTML]{B2BEB5} & \cellcolor[HTML]{B2BEB5} & \cellcolor[HTML]{B2BEB5} & \cellcolor[HTML]{B2BEB5} & \cellcolor[HTML]{B2BEB5} & \multicolumn{1}{c|}{\cellcolor[HTML]{B2BEB5}} & \multicolumn{1}{c|}{\scriptsize{$\bullet$}} & \scriptsize{$\circ$} & \scriptsize{$\bullet$} \\ \cline{1-1} \cline{12-14} 
        $m$ & \cellcolor[HTML]{B2BEB5} & \cellcolor[HTML]{B2BEB5} & \cellcolor[HTML]{B2BEB5} & \cellcolor[HTML]{B2BEB5} & \cellcolor[HTML]{B2BEB5} & \cellcolor[HTML]{B2BEB5} & \cellcolor[HTML]{B2BEB5} & \cellcolor[HTML]{B2BEB5} & \cellcolor[HTML]{B2BEB5} & \cellcolor[HTML]{B2BEB5} & \multicolumn{1}{c|}{\cellcolor[HTML]{B2BEB5}} & \scriptsize{$\bullet$} & \scriptsize{$\circ$} \\ \cline{1-1} \cline{13-14} 
        $n$ & \cellcolor[HTML]{B2BEB5} & \cellcolor[HTML]{B2BEB5} & \cellcolor[HTML]{B2BEB5} & \cellcolor[HTML]{B2BEB5} & \cellcolor[HTML]{B2BEB5} & \cellcolor[HTML]{B2BEB5} & \cellcolor[HTML]{B2BEB5} & \cellcolor[HTML]{B2BEB5} & \cellcolor[HTML]{B2BEB5} & \cellcolor[HTML]{B2BEB5} & \cellcolor[HTML]{B2BEB5} & \cellcolor[HTML]{B2BEB5} & \scriptsize{$\bullet$} \\ \hline
        \end{tabular}
        \end{center}
        \caption{Tabular representation of the entire interaction.        
        \label{tab:scenario1}}
        \end{table}
        
        Given an event and a state traces~$\tau_{\chi}^e$ and~$\tau_{\chi}^s$, we propose to display the consecutive states, showing fluents as hexagons and the triggered events as rectangles. Since the acceptability of arguments is 
        what mainly matters, we propose to represent only the fluents $a_x$, using the argument names for the sake of readability.  
        Moreover, we do not show fluents when their negation is true in the state, except when the occurrence of a represented event results in the negation of the fluent. In this case, the negation is represented by a lighter shade. The events~$enunciate_x$, $makesUnacc_{y,x}$, and~$makesAcc_x$ are shortened as~$enu_x$, $una_{y,x}$, and~$acc_x$, respectively.

        \addtocounter{example}{-1}
        \begin{example} (continued) --
            Figure~\ref{fig:exemple_ini} shows a partial representation of the state trace obtained for Example~\ref{ex:IRM_ou_radio} using the ASP implementation described in Section~\ref{sec:ASP}. 
            
            The first represented state corresponds to~$S(6)$, an argumentative state in the sense of Def.~\ref{def:admissible_state}, which allows for the enunciation of the next argument: since all arguments preceding~$e$ have already been enunciated, the action~$enunciate_e$ can be performed. The occurrence of this event is the transition to the next state~$S(7)$ where, as shown in Figure~\ref{fig:exemple_ini}, argument~$e$ is acceptable. Unlike~$S(6)$, $S(7)$ is not an argumentative state: condition (i) of Def.~\ref{def:admissible_state} is not satisfied because~${(a_d\wedge cA_{e,d})}$ and~${a_e\in S(7)}$. Therefore, the next argument cannot be enunciated. However, since the triggering conditions of~$makesUnacc_{e,d}$ are satisfied, this exogenous event is triggered, leading to a new state transition. Since argument $d$ is no longer acceptable in~$S(8)$, condition (i)  of Def.~\ref{def:admissible_state} is now satisfied. Still, condition (ii) is not satisfied by~$S(8)$, preventing the next argument from being enunciated. Instead,~$makesAcc_c$ is triggered, leading to the following state~$S(9)$. Here, as shown in Figure~\ref{fig:exemple_ini}, argument~$c$ is acceptable. As this new state is argumentative, the next argument,~$f$, can be stated. The dialogue continues step by step and ends at state~$S(31)$. 
        \end{example}

        A second, more compact, tabular visualisation is proposed, illustrated in Table~\ref{tab:scenario1}: the arguments are represented in the first column, the order of the performed actions in the first row. For the sake of readability,  $enunciate_x$ is shortened as~$x$. In each table cell, $\bullet$ means that the argument is acceptable while~$\circ$ means that it is not. If an argument has not been enunciated yet, its acceptability cannot be evaluated, which is represented by the shaded boxes. In contrast to the previous representation where the updating stages are shown, this second form has the advantage of being more compact and allows the display the whole dialogue. It also makes it possible to  see quickly the direct and indirect impacts of the argument enunciation on the other argument acceptability. In particular, the enunciation order effect can be observed, as illustrated by the graphical comparison of Example~\ref{ex:IRM_ou_radio} and its modification given in Example~\ref{ex:IRM_ou_radio_modif}.
        
        \begin{example}\label{ex:IRM_ou_radio_modif}
          Let us consider the same dialogue as in Example~\ref{ex:IRM_ou_radio}, starting with the enunciation of arguments $a,b,c$, but considering that the physician then directly asks if it is possible to do the MRI today ($l$). The radiologist replies that he can only do it in two days at the earliest~($m$). The physician then specifies that it is an emergency~($n$). The remaining arguments are then enunciated in the same order ad in the initial example.
          
          Table~\ref{tab:scenario2} displays the proposed compact visualisation  for the evolution of the acceptability of the decision variable~$c$, starting from its enunciation. Even if the final state of the argumentation graph is identical, as expected according to Proposition~\ref{prop:indep_order} established in the previous section,
          with~$c$ being rejected, the display makes it easy to observe the very important impact that the order of the actions can have on the intermediate stages that lead to it: in the new scenario,~$c$ is not accepted from the $6^{th}$~action, i.e. $n$~enunciation, with no modification until the end.  
        \end{example}
        
        This visualisation thus also illustrates the relevance of the temporality integration in the argumentation framework: the differences between the two scenarios cannot be captured by classical AAFs. 
        
        \begin{table}[t]
            \addtolength{\tabcolsep}{-1pt}
            \begin{center}
                \begin{tabular}{|c|c|c|c|c|c|c|c|c|c|c|c|}
                \hline
                $\varsigma_1$ &  c & d & e & f & g & h,i & j & k & l & m & n \\ \hline
                c & \scriptsize{$\bullet$} & \scriptsize{$\circ$} & \scriptsize{$\bullet$} & \scriptsize{$\circ$} & \scriptsize{$\bullet$} & \scriptsize{$\circ$} & \scriptsize{$\circ$} & \scriptsize{$\bullet$} & \scriptsize{$\bullet$} & \scriptsize{$\bullet$} & \scriptsize{$\circ$} \\ 
                \hhline{|=|=|=|=|=|=|=|=|=|=|=|=|} 
                $\varsigma_2$ & c & l & m & n & d & e & f & g & h,i & j & k \\ \hline
                c & \scriptsize{$\bullet$} & \scriptsize{$\bullet$} & \scriptsize{$\bullet$} & \scriptsize{$\circ$} & \scriptsize{$\circ$} & \scriptsize{$\circ$} & \scriptsize{$\circ$} & \scriptsize{$\circ$} & \scriptsize{$\circ$} & \scriptsize{$\circ$} & \scriptsize{$\circ$} \\ \hline
                \end{tabular}
            \end{center}           
            \caption{Impact of the order in which arguments are enunciated~(lines 1, 3) on the acceptability of the arguments~(lines 2, 4).\vspace{-2mm}}
            \label{tab:scenario2}
        \end{table}
        
    \subsection{On Causality and Explanation}
    \label{sec:discu_causalite}
    
    Beyond the graphical representation of the acceptance/rejection process, the proposed formalisation of AAF into action models provides tools for richer explanations, allowing us to transfer the notion of actual causality recalled in Section~\ref{sec:causality} to the argumentation framework. Indeed, the extraction of causal chains  has been shown to be an important property for explanations~\cite{miller_explanation_2018}. In the taxonomy proposed in the case of argumentation in \cite{cyras_survey}, such a causal explanation can be related to the identification of arguments that must be removed from an argumentation graph to make a non-acceptable argument acceptable~\cite{fan2015explanations}. In causal terminology, this corresponds to the search for a \textit{but-for} cause of the non-acceptability of an argument. 
    
    However, this test does not solve cases where the occurrence of one of two events would have been sufficient to cause an effect in the absence of the other, called over-determination~\cite{menzies_counterfactual_2020}. Among others, the definition of causality underlying the NESS test, as briefly recalled in Section~\ref{sec:causality} and implemented for the considered action language, makes it possible to solve this issue.

    Structural equations~\cite{halpern2005causes} constitute another formal model of causality that addresses the over-determination issue and can be exploited in the argumentation framework, using the transformation of acyclic abstract argumentation graphs to that formalism proposed in~\cite{munro2022argumentation}. The main differences are as follows: from a philosophical point of view, the definition of causality underlying the NESS test belongs to the family of regularity approaches~\cite{andreas_regularity_2021}, whereas Halpern's definitions belong to the family of counterfactual approaches~\cite{menzies_counterfactual_2020}. Secondly, the use of action languages makes it possible to model and take into account temporality and the dynamics of the dialogue, which is a crucial component. From a mathematical point of view, according to~\cite{beckers_causal_2021}, Halpern's definition of causality can be described as `Contrastive actual weak sufficiency', whereas the one used here would be `Minimal actual strong sufficiency': in a nutshell, whereas the former emphasises that a cause must be necessary for an effect, hence the contrastive aspect, the latter emphasises sufficiency and subordinates necessity to it. From a practical point of view, the advantage of the causal approach used here is that it does not require counterfactual reasoning or interventionism, mechanisms that are computationally onerous and criticised for introducing subjectivity into causal enquiry~\cite{sarmiento_action_2022-1,wright_causation_1985}.

        These causal relations, that may lead later to causal explanations,
        can be represented graphically, enriching the proposed  visualisation illustrated in Figure~\ref{fig:exemple_ini} by different types of causes. This principle is illustrated in Figure~\ref{fig:exemple_fin}, commented below. 
        
        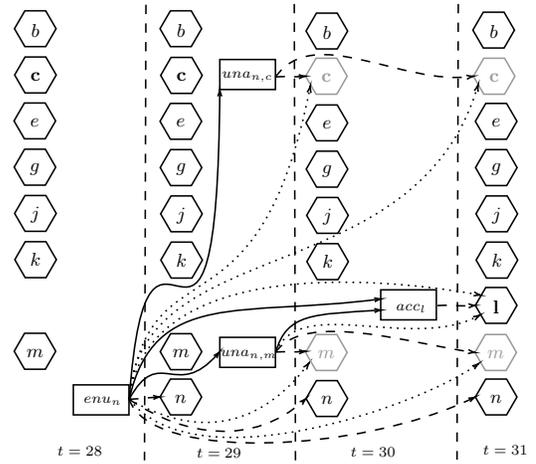
\begin{figure}[t]
            \centering
        	\resizebox{7cm}{!}{\input{fig_exemple_fin}}
        	\caption{Enriched graphical representation of Example~\ref{ex:IRM_ou_radio}, with causal relations extracts:  Direct NESS-causes \DirectNESS, NESS-causes \NESS, and actual causes \actual.}
        	\label{fig:exemple_fin}
        \end{figure}
        
        \addtocounter{example}{-2}
        \begin{example} (continued) --
 Figure~\ref{fig:exemple_fin} graphically displays 
 the last four states in the trace of  Example~\ref{ex:IRM_ou_radio},  corresponding to the enunciation of argument~$n$ and the subsequent update mechanisms.           
Argument~$n$, which states the urgency of the examination,
is the one that closes the debate. As represented in Figure~\ref{fig:exemple_fin}  its enunciation in state~$S(28)$ is a direct NESS-cause (dNc) of its acceptability in the following states, a relation we denote by~$(enunciate_n,28)$ dNc $(a_n,29-31)$. Similarly, we have $(makesUnacc_{n,c},29)$ dNc $(\neg a_c,{30-31})$, $(makesUnacc_{n,m},29)$ dNc ${(\neg a_m,30-31)}$, and $(makesAcc_l,30)$ dNc $(a_l,31)$. As these examples show, this first relationship is the basic building block of causality, which is concerned with causal relationships given the actual effects of the occurrence of an event. Yet this relationship is not enough. If we want to know why argument~$l$  is acceptable at the end of the dialogue (i.e. why the decision to have an MRI on the same day is made), it is not satisfactory to simply say that it is because of the event~$(makesAcc_l,30)$.
           
To find out why the latter happens, we need to look at the NESS causes and the actual causes to construct the causal chain that lead to it. By transitivity we get that $(makesUnacc_{n,m},29)$ is a cause of the fact that $makesAcc_l$ was triggered, and therefore of the effects that triggering may have had. Going back even further and looking for the causes for which the occurrence~$(makesUnacc_{n,m},29)$ took place, we find $(enunciate_n,28)$ actual cause~$(makesUnacc_{n,m},29)$ and therefore $(enunciate_n,28)$ NESS-cause~$(\neg a_m,30-31)$. By transitivity we can derive $(enunciate_n,28)$ NESS-cause~$(a_l,31)$. This new relation allows us to say that the physician enunciating that it is an emergency is one of the causes of the final decision, an answer that already seems more satisfactory and can be included in an explanation. The same reasoning can be applied to find the causes of $(\neg a_c,31)$, the other decision variable.
        \end{example}


%% file: fig_exemple_ini.tex
\tikzset{every picture/.style={line width=0.75pt}} 

\begin{tikzpicture}[x=0.75pt,y=0.75pt,yscale=-1,xscale=1]

\draw  [line width=0.75]  (38.8,129.2) -- (74.6,129.2) -- (74.6,148.54) -- (38.8,148.54) -- cycle ;
\draw  [dash pattern={on 4.5pt off 4.5pt}]  (85.53,6.3) -- (85,191.83) ;
\draw  [dash pattern={on 4.5pt off 4.5pt}]  (181.03,5.8) -- (180.67,191.83) ;
\draw  [dash pattern={on 4.5pt off 4.5pt}]  (286.6,5.8) -- (286.67,191.83) ;
\draw  [line width=0.75]  (133.09,90.57) -- (168.89,90.57) -- (168.89,109.91) -- (133.09,109.91) -- cycle ;
\draw   (27.7,21.73) -- (21,33.26) -- (7.6,33.26) -- (0.91,21.73) -- (7.6,10.19) -- (21,10.19) -- cycle ;
\draw   (121.4,21.73) -- (114.7,33.26) -- (101.3,33.26) -- (94.6,21.73) -- (101.3,10.19) -- (114.7,10.19) -- cycle ;
\draw   (215.4,23.23) -- (208.7,34.76) -- (195.3,34.76) -- (188.6,23.23) -- (195.3,11.69) -- (208.7,11.69) -- cycle ;
\draw   (321.2,61.23) -- (314.5,72.76) -- (301.1,72.76) -- (294.4,61.23) -- (301.1,49.69) -- (314.5,49.69) -- cycle ;
\draw   (323.5,139.53) -- (316.81,151.12) -- (303.43,151.12) -- (296.74,139.53) -- (303.43,127.95) -- (316.81,127.95) -- cycle ;
\draw   (218.5,139.53) -- (211.81,151.12) -- (198.43,151.12) -- (191.74,139.53) -- (198.43,127.95) -- (211.81,127.95) -- cycle ;
\draw   (122.9,100.73) -- (116.21,112.32) -- (102.83,112.32) -- (96.14,100.73) -- (102.83,89.15) -- (116.21,89.15) -- cycle ;
\draw   (123.5,139.53) -- (116.81,151.12) -- (103.43,151.12) -- (96.74,139.53) -- (103.43,127.95) -- (116.81,127.95) -- cycle ;
\draw   (28.9,100.73) -- (22.21,112.32) -- (8.83,112.32) -- (2.14,100.73) -- (8.83,89.15) -- (22.21,89.15) -- cycle ;
\draw  [line width=0.75]  (228.69,50.97) -- (264.49,50.97) -- (264.49,70.31) -- (228.69,70.31) -- cycle ;
\draw  [line width=0.75]  (333.09,160.97) -- (368.89,160.97) -- (368.89,180.31) -- (333.09,180.31) -- cycle ;
\draw  [color={rgb, 255:red, 155; green, 155; blue, 155 }  ,draw opacity=1 ] (218.7,100.53) -- (212.01,112.12) -- (198.63,112.12) -- (191.94,100.53) -- (198.63,88.95) -- (212.01,88.95) -- cycle ;
\draw  [color={rgb, 255:red, 155; green, 155; blue, 155 }  ,draw opacity=1 ] (321.3,101.13) -- (314.61,112.72) -- (301.23,112.72) -- (294.54,101.13) -- (301.23,89.55) -- (314.61,89.55) -- cycle ;
\draw   (321.4,23.23) -- (314.7,34.76) -- (301.3,34.76) -- (294.6,23.23) -- (301.3,11.69) -- (314.7,11.69) -- cycle ;

\draw (27.13,180.4) node [anchor=north west][inner sep=0.75pt]  [font=\scriptsize]  {$t=6$};
\draw (44,135.6) node [anchor=north west][inner sep=0.75pt]  [font=\scriptsize]  {$enu_{e}$};
\draw (135.09,96.97) node [anchor=north west][inner sep=0.75pt]  [font=\scriptsize]  {$una_{e,d}$};
\draw (116.67,181.33) node [anchor=north west][inner sep=0.75pt]  [font=\scriptsize]  {$t=7$};
\draw (215.8,180.9) node [anchor=north west][inner sep=0.75pt]  [font=\scriptsize]  {$t=8$};
\draw (301,179.9) node [anchor=north west][inner sep=0.75pt]  [font=\scriptsize]  {$t=9$};
\draw (14.3,22.73) node  [font=\small]  {$b$};
\draw (108,22.73) node  [font=\small]  {$b$};
\draw (202,24.23) node  [font=\small]  {$b$};
\draw (307.8,62.23) node  [font=\small]  {$\mathbf{c}$};
\draw (310.12,140.53) node  [font=\small]  {$e$};
\draw (205.12,140.53) node  [font=\small]  {$e$};
\draw (109.52,101.73) node  [font=\small]  {$d$};
\draw (110.12,140.53) node  [font=\small]  {$e$};
\draw (15.52,101.73) node  [font=\small]  {$d$};
\draw (236.69,58.57) node [anchor=north west][inner sep=0.75pt]  [font=\scriptsize]  {$acc_{c}$};
\draw (339.09,167.57) node [anchor=north west][inner sep=0.75pt]  [font=\scriptsize]  {$enu_{f}$};
\draw (205.32,101.53) node  [font=\small,color={rgb, 255:red, 155; green, 155; blue, 155 }  ,opacity=1 ]  {$d$};
\draw (307.92,102.13) node  [font=\small,color={rgb, 255:red, 155; green, 155; blue, 155 }  ,opacity=1 ]  {$d$};
\draw (308,24.23) node  [font=\small]  {$b$};

\end{tikzpicture}

%% file: fig_exemple_fin.tex
\tikzset{every picture/.style={line width=0.75pt}} 

\begin{tikzpicture}[x=0.75pt,y=0.75pt,yscale=-1,xscale=1]

\draw  [line width=0.75]  (38.8,251.2) -- (74.6,251.2) -- (74.6,270.54) -- (38.8,270.54) -- cycle ;
\draw  [dash pattern={on 4.5pt off 4.5pt}]  (85.53,6.3) -- (84.6,299.8) ;
\draw  [dash pattern={on 4.5pt off 4.5pt}]  (181.03,5.8) -- (181.8,299.8) ;
\draw  [dash pattern={on 4.5pt off 4.5pt}]  (286.6,5.8) -- (285.8,299.4) ;
\draw  [line width=0.75]  (133.09,41.57) -- (168.89,41.57) -- (168.89,60.91) -- (133.09,60.91) -- cycle ;
\draw   (27.7,21.73) -- (21,33.26) -- (7.6,33.26) -- (0.91,21.73) -- (7.6,10.19) -- (21,10.19) -- cycle ;
\draw   (121.4,21.73) -- (114.7,33.26) -- (101.3,33.26) -- (94.6,21.73) -- (101.3,10.19) -- (114.7,10.19) -- cycle ;
\draw   (215.4,23.23) -- (208.7,34.76) -- (195.3,34.76) -- (188.6,23.23) -- (195.3,11.69) -- (208.7,11.69) -- cycle ;
\draw  [line width=0.75]  (236.69,189.97) -- (272.49,189.97) -- (272.49,209.31) -- (236.69,209.31) -- cycle ;
\draw  [dash pattern={on 4.5pt off 4.5pt}]  (74.74,260.46) .. controls (110.24,284.75) and (154.9,290.88) .. (187.62,260.92) ;
\draw [shift={(188.6,260)}, rotate = 136.62] [color={rgb, 255:red, 0; green, 0; blue, 0 }  ][line width=0.75]    (4.37,-1.32) .. controls (2.78,-0.56) and (1.32,-0.12) .. (0,0) .. controls (1.32,0.12) and (2.78,0.56) .. (4.37,1.32)   ;
\draw    (74.74,260.46) .. controls (94.3,224.74) and (104.69,262.24) .. (131.75,231.26) ;
\draw [shift={(133,229.8)}, rotate = 129.81] [color={rgb, 255:red, 0; green, 0; blue, 0 }  ][line width=0.75]    (4.37,-1.32) .. controls (2.78,-0.56) and (1.32,-0.12) .. (0,0) .. controls (1.32,0.12) and (2.78,0.56) .. (4.37,1.32)   ;
\draw  [color={rgb, 255:red, 155; green, 155; blue, 155 }  ,draw opacity=1 ] (215.1,52.33) -- (208.41,63.92) -- (195.03,63.92) -- (188.34,52.33) -- (195.03,40.75) -- (208.41,40.75) -- cycle ;
\draw  [dash pattern={on 0.84pt off 2.51pt}]  (74.74,260.46) .. controls (136.15,295.88) and (173.1,260.09) .. (189.61,238.31) ;
\draw [shift={(190.6,237)}, rotate = 126.53] [color={rgb, 255:red, 0; green, 0; blue, 0 }  ][line width=0.75]    (4.37,-1.32) .. controls (2.78,-0.56) and (1.32,-0.12) .. (0,0) .. controls (1.32,0.12) and (2.78,0.56) .. (4.37,1.32)   ;
\draw   (27.78,51.53) -- (21.09,63.12) -- (7.71,63.12) -- (1.02,51.53) -- (7.71,39.95) -- (21.09,39.95) -- cycle ;
\draw   (27.4,81.13) -- (20.7,92.66) -- (7.3,92.66) -- (0.6,81.13) -- (7.3,69.59) -- (20.7,69.59) -- cycle ;
\draw   (27.48,110.93) -- (20.79,122.52) -- (7.41,122.52) -- (0.72,110.93) -- (7.41,99.35) -- (20.79,99.35) -- cycle ;
\draw   (27.8,140.73) -- (21.1,152.26) -- (7.7,152.26) -- (1,140.73) -- (7.7,129.19) -- (21.1,129.19) -- cycle ;
\draw   (27.88,170.53) -- (21.19,182.12) -- (7.81,182.12) -- (1.12,170.53) -- (7.81,158.95) -- (21.19,158.95) -- cycle ;
\draw   (27.48,229.53) -- (20.79,241.12) -- (7.41,241.12) -- (0.72,229.53) -- (7.41,217.95) -- (20.79,217.95) -- cycle ;
\draw   (121.78,51.8) -- (115.09,63.39) -- (101.71,63.39) -- (95.02,51.8) -- (101.71,40.22) -- (115.09,40.22) -- cycle ;
\draw   (121.4,81.4) -- (114.7,92.94) -- (101.3,92.94) -- (94.6,81.4) -- (101.3,69.86) -- (114.7,69.86) -- cycle ;
\draw   (121.48,111.2) -- (114.79,122.79) -- (101.41,122.79) -- (94.72,111.2) -- (101.41,99.62) -- (114.79,99.62) -- cycle ;
\draw   (121.8,141) -- (115.1,152.54) -- (101.7,152.54) -- (95,141) -- (101.7,129.46) -- (115.1,129.46) -- cycle ;
\draw   (121.88,170.8) -- (115.19,182.39) -- (101.81,182.39) -- (95.12,170.8) -- (101.81,159.22) -- (115.19,159.22) -- cycle ;
\draw   (121.48,229.8) -- (114.79,241.39) -- (101.41,241.39) -- (94.72,229.8) -- (101.41,218.22) -- (114.79,218.22) -- cycle ;
\draw   (121.4,259) -- (114.7,270.54) -- (101.3,270.54) -- (94.6,259) -- (101.3,247.46) -- (114.7,247.46) -- cycle ;
\draw   (215.4,82.4) -- (208.7,93.94) -- (195.3,93.94) -- (188.6,82.4) -- (195.3,70.86) -- (208.7,70.86) -- cycle ;
\draw   (215.48,112.2) -- (208.79,123.79) -- (195.41,123.79) -- (188.72,112.2) -- (195.41,100.62) -- (208.79,100.62) -- cycle ;
\draw   (215.8,142) -- (209.1,153.54) -- (195.7,153.54) -- (189,142) -- (195.7,130.46) -- (209.1,130.46) -- cycle ;
\draw   (215.88,171.8) -- (209.19,183.39) -- (195.81,183.39) -- (189.12,171.8) -- (195.81,160.22) -- (209.19,160.22) -- cycle ;
\draw   (215.4,260) -- (208.7,271.54) -- (195.3,271.54) -- (188.6,260) -- (195.3,248.46) -- (208.7,248.46) -- cycle ;
\draw   (324.4,81.4) -- (317.7,92.94) -- (304.3,92.94) -- (297.6,81.4) -- (304.3,69.86) -- (317.7,69.86) -- cycle ;
\draw   (324.48,111.2) -- (317.79,122.79) -- (304.41,122.79) -- (297.72,111.2) -- (304.41,99.62) -- (317.79,99.62) -- cycle ;
\draw   (324.8,141) -- (318.1,152.54) -- (304.7,152.54) -- (298,141) -- (304.7,129.46) -- (318.1,129.46) -- cycle ;
\draw   (324.88,170.8) -- (318.19,182.39) -- (304.81,182.39) -- (298.12,170.8) -- (304.81,159.22) -- (318.19,159.22) -- cycle ;
\draw   (324.4,200) -- (317.7,211.54) -- (304.3,211.54) -- (297.6,200) -- (304.3,188.46) -- (317.7,188.46) -- cycle ;
\draw   (324.4,259) -- (317.7,270.54) -- (304.3,270.54) -- (297.6,259) -- (304.3,247.46) -- (317.7,247.46) -- cycle ;
\draw  [line width=0.75]  (133.09,220.57) -- (168.89,220.57) -- (168.89,239.91) -- (133.09,239.91) -- cycle ;
\draw  [color={rgb, 255:red, 155; green, 155; blue, 155 }  ,draw opacity=1 ] (323.1,52.33) -- (316.41,63.92) -- (303.03,63.92) -- (296.34,52.33) -- (303.03,40.75) -- (316.41,40.75) -- cycle ;
\draw  [color={rgb, 255:red, 155; green, 155; blue, 155 }  ,draw opacity=1 ] (215.1,230.33) -- (208.41,241.92) -- (195.03,241.92) -- (188.34,230.33) -- (195.03,218.75) -- (208.41,218.75) -- cycle ;
\draw  [color={rgb, 255:red, 155; green, 155; blue, 155 }  ,draw opacity=1 ] (324.1,230.33) -- (317.41,241.92) -- (304.03,241.92) -- (297.34,230.33) -- (304.03,218.75) -- (317.41,218.75) -- cycle ;
\draw    (74.74,260.46) .. controls (87.4,89) and (131.4,305) .. (133.09,60.91) ;
\draw [shift={(133.09,60.91)}, rotate = 90.4] [color={rgb, 255:red, 0; green, 0; blue, 0 }  ][line width=0.75]    (4.37,-1.32) .. controls (2.78,-0.56) and (1.32,-0.12) .. (0,0) .. controls (1.32,0.12) and (2.78,0.56) .. (4.37,1.32)   ;
\draw    (74.74,260.46) .. controls (90.92,188.96) and (147.23,207.61) .. (235.66,195.98) ;
\draw [shift={(237,195.8)}, rotate = 172.34] [color={rgb, 255:red, 0; green, 0; blue, 0 }  ][line width=0.75]    (4.37,-1.32) .. controls (2.78,-0.56) and (1.32,-0.12) .. (0,0) .. controls (1.32,0.12) and (2.78,0.56) .. (4.37,1.32)   ;
\draw    (169.14,230.06) .. controls (176.09,204.98) and (193.66,205.77) .. (235.09,203.12) ;
\draw [shift={(237,203)}, rotate = 176.26] [color={rgb, 255:red, 0; green, 0; blue, 0 }  ][line width=0.75]    (4.37,-1.32) .. controls (2.78,-0.56) and (1.32,-0.12) .. (0,0) .. controls (1.32,0.12) and (2.78,0.56) .. (4.37,1.32)   ;
\draw  [dash pattern={on 4.5pt off 4.5pt}]  (74.74,260.46) .. controls (88.13,300.8) and (222.84,295.45) .. (296.5,259.54) ;
\draw [shift={(297.6,259)}, rotate = 153.62] [color={rgb, 255:red, 0; green, 0; blue, 0 }  ][line width=0.75]    (4.37,-1.32) .. controls (2.78,-0.56) and (1.32,-0.12) .. (0,0) .. controls (1.32,0.12) and (2.78,0.56) .. (4.37,1.32)   ;
\draw  [dash pattern={on 4.5pt off 4.5pt}]  (74.74,260.46) .. controls (88.36,259.78) and (80.46,259.9) .. (92.73,259.12) ;
\draw [shift={(94.6,259)}, rotate = 176.49] [color={rgb, 255:red, 0; green, 0; blue, 0 }  ][line width=0.75]    (4.37,-1.32) .. controls (2.78,-0.56) and (1.32,-0.12) .. (0,0) .. controls (1.32,0.12) and (2.78,0.56) .. (4.37,1.32)   ;
\draw  [dash pattern={on 4.5pt off 4.5pt}]  (169,52.78) .. controls (182.62,52.1) and (174.24,53.15) .. (186.47,52.44) ;
\draw [shift={(188.34,52.33)}, rotate = 176.49] [color={rgb, 255:red, 0; green, 0; blue, 0 }  ][line width=0.75]    (4.37,-1.32) .. controls (2.78,-0.56) and (1.32,-0.12) .. (0,0) .. controls (1.32,0.12) and (2.78,0.56) .. (4.37,1.32)   ;
\draw  [dash pattern={on 4.5pt off 4.5pt}]  (169.14,230.06) .. controls (182.76,229.38) and (174.38,230.43) .. (186.61,229.73) ;
\draw [shift={(188.48,229.62)}, rotate = 176.49] [color={rgb, 255:red, 0; green, 0; blue, 0 }  ][line width=0.75]    (4.37,-1.32) .. controls (2.78,-0.56) and (1.32,-0.12) .. (0,0) .. controls (1.32,0.12) and (2.78,0.56) .. (4.37,1.32)   ;
\draw  [dash pattern={on 4.5pt off 4.5pt}]  (169.14,230.06) .. controls (187.91,203.01) and (256.5,224.73) .. (295.58,230.1) ;
\draw [shift={(297.34,230.33)}, rotate = 187.26] [color={rgb, 255:red, 0; green, 0; blue, 0 }  ][line width=0.75]    (4.37,-1.32) .. controls (2.78,-0.56) and (1.32,-0.12) .. (0,0) .. controls (1.32,0.12) and (2.78,0.56) .. (4.37,1.32)   ;
\draw  [dash pattern={on 4.5pt off 4.5pt}]  (169,52.78) .. controls (193.95,17.75) and (250.26,58.57) .. (294.99,52.53) ;
\draw [shift={(296.34,52.33)}, rotate = 171.06] [color={rgb, 255:red, 0; green, 0; blue, 0 }  ][line width=0.75]    (4.37,-1.32) .. controls (2.78,-0.56) and (1.32,-0.12) .. (0,0) .. controls (1.32,0.12) and (2.78,0.56) .. (4.37,1.32)   ;
\draw  [dash pattern={on 4.5pt off 4.5pt}]  (273.22,200.11) .. controls (286.84,199.43) and (283.06,200.78) .. (295.7,200.11) ;
\draw [shift={(297.6,200)}, rotate = 176.49] [color={rgb, 255:red, 0; green, 0; blue, 0 }  ][line width=0.75]    (4.37,-1.32) .. controls (2.78,-0.56) and (1.32,-0.12) .. (0,0) .. controls (1.32,0.12) and (2.78,0.56) .. (4.37,1.32)   ;
\draw  [dash pattern={on 0.84pt off 2.51pt}]  (74.74,260.46) .. controls (118.16,315.64) and (262.47,264.45) .. (298.73,238.19) ;
\draw [shift={(299.8,237.4)}, rotate = 142.92] [color={rgb, 255:red, 0; green, 0; blue, 0 }  ][line width=0.75]    (4.37,-1.32) .. controls (2.78,-0.56) and (1.32,-0.12) .. (0,0) .. controls (1.32,0.12) and (2.78,0.56) .. (4.37,1.32)   ;
\draw  [dash pattern={on 0.84pt off 2.51pt}]  (74.74,260.46) .. controls (87.54,158.68) and (268.35,187.37) .. (299.3,193.08) ;
\draw [shift={(301,193.4)}, rotate = 190.78] [color={rgb, 255:red, 0; green, 0; blue, 0 }  ][line width=0.75]    (4.37,-1.32) .. controls (2.78,-0.56) and (1.32,-0.12) .. (0,0) .. controls (1.32,0.12) and (2.78,0.56) .. (4.37,1.32)   ;
\draw  [dash pattern={on 0.84pt off 2.51pt}]  (74.74,260.46) .. controls (87.8,155) and (157.8,229.8) .. (191.4,58.6) ;
\draw [shift={(191.4,58.6)}, rotate = 101.1] [color={rgb, 255:red, 0; green, 0; blue, 0 }  ][line width=0.75]    (4.37,-1.32) .. controls (2.78,-0.56) and (1.32,-0.12) .. (0,0) .. controls (1.32,0.12) and (2.78,0.56) .. (4.37,1.32)   ;
\draw  [dash pattern={on 0.84pt off 2.51pt}]  (74.74,260.46) .. controls (81.4,151.4) and (260.2,192.6) .. (299.4,57.8) ;
\draw [shift={(299.4,57.8)}, rotate = 106.21] [color={rgb, 255:red, 0; green, 0; blue, 0 }  ][line width=0.75]    (4.37,-1.32) .. controls (2.78,-0.56) and (1.32,-0.12) .. (0,0) .. controls (1.32,0.12) and (2.78,0.56) .. (4.37,1.32)   ;
\draw  [dash pattern={on 0.84pt off 2.51pt}]  (169.14,230.06) .. controls (184.45,192.58) and (255.95,229.84) .. (299.3,206.91) ;
\draw [shift={(300.6,206.2)}, rotate = 150.54] [color={rgb, 255:red, 0; green, 0; blue, 0 }  ][line width=0.75]    (4.37,-1.32) .. controls (2.78,-0.56) and (1.32,-0.12) .. (0,0) .. controls (1.32,0.12) and (2.78,0.56) .. (4.37,1.32)   ;
\draw   (322.8,22.63) -- (316.1,34.16) -- (302.7,34.16) -- (296,22.63) -- (302.7,11.09) -- (316.1,11.09) -- cycle ;

\draw (27.13,289.4) node [anchor=north west][inner sep=0.75pt]  [font=\scriptsize]  {$t=28$};
\draw (44,257.6) node [anchor=north west][inner sep=0.75pt]  [font=\scriptsize]  {$enu_{n}$};
\draw (133.09,47.97) node [anchor=north west][inner sep=0.75pt]  [font=\scriptsize]  {$una_{n,c}$};
\draw (116.67,290.33) node [anchor=north west][inner sep=0.75pt]  [font=\scriptsize]  {$t=29$};
\draw (215.8,289.9) node [anchor=north west][inner sep=0.75pt]  [font=\scriptsize]  {$t=30$};
\draw (301,288.9) node [anchor=north west][inner sep=0.75pt]  [font=\scriptsize]  {$t=31$};
\draw (14.3,22.73) node  [font=\small]  {$b$};
\draw (108,22.73) node  [font=\small]  {$b$};
\draw (202,24.23) node  [font=\small]  {$b$};
\draw (245.19,197.57) node [anchor=north west][inner sep=0.75pt]  [font=\scriptsize]  {$acc_{l}$};
\draw (201.72,53.33) node  [font=\small,color={rgb, 255:red, 155; green, 155; blue, 155 }  ,opacity=1 ]  {$\mathbf{c}$};
\draw (14.4,52.53) node  [font=\small]  {$\mathbf{c}$};
\draw (14,82.13) node  [font=\small]  {$e$};
\draw (14.1,111.93) node  [font=\small]  {$g$};
\draw (14.4,141.73) node  [font=\small]  {$j$};
\draw (14.5,170.53) node  [font=\small]  {$k$};
\draw (14.1,230.53) node  [font=\small]  {$m$};
\draw (108.4,52.8) node  [font=\small]  {$\mathbf{c}$};
\draw (108,82.4) node  [font=\small]  {$e$};
\draw (108.1,112.2) node  [font=\small]  {$g$};
\draw (108.4,142) node  [font=\small]  {$j$};
\draw (108.5,170.8) node  [font=\small]  {$k$};
\draw (108.1,230.8) node  [font=\small]  {$m$};
\draw (108,260) node  [font=\small]  {$n$};
\draw (202,83.4) node  [font=\small]  {$e$};
\draw (202.1,113.2) node  [font=\small]  {$g$};
\draw (202.4,143) node  [font=\small]  {$j$};
\draw (202.5,171.8) node  [font=\small]  {$k$};
\draw (202,261) node  [font=\small]  {$n$};
\draw (311,82.4) node  [font=\small]  {$e$};
\draw (311.1,112.2) node  [font=\small]  {$g$};
\draw (311.4,142) node  [font=\small]  {$j$};
\draw (311.5,170.8) node  [font=\small]  {$k$};
\draw (311,201) node  [font=\small]  {$\mathbf{l}$};
\draw (311,260) node  [font=\small]  {$n$};
\draw (132.09,226.97) node [anchor=north west][inner sep=0.75pt]  [font=\scriptsize]  {$una_{n,m}$};
\draw (309.72,53.33) node  [font=\small,color={rgb, 255:red, 155; green, 155; blue, 155 }  ,opacity=1 ]  {$\mathbf{c}$};
\draw (201.72,231.33) node  [font=\small,color={rgb, 255:red, 155; green, 155; blue, 155 }  ,opacity=1 ]  {$m$};
\draw (310.72,231.33) node  [font=\small,color={rgb, 255:red, 155; green, 155; blue, 155 }  ,opacity=1 ]  {$m$};
\draw (309.4,23.63) node  [font=\small]  {$b$};

\end{tikzpicture}

%% file: 0.3_conclusion.tex
\section{Conclusion}
\label{sec:conclusion}

This paper has proposed a formalisation of acyclic abstract argumentation systems in the action language of~\cite{sarmiento_action_2022-1}, establishing its formal properties: it first allows increasing the expressiveness of these models, through the integration of temporality, making it possible to examine the effect of the order of the argument enunciation. Moreover, it allows us to exploit the notion of causality associated to the action language, offering the possibility to give rich information about the argument acceptance or rejection and  justifications about the latter. The paper has proposed two types of graphical representations of the argumentation process that can be used as visual support, opening the way for new forms of argumentation explanations. 

Future works will aim at developing such explanations, applying the principles developed  in the context of eXplainable Artificial Intelligence (XAI), e.g. detailed in \cite{miller_explanation_2018}: causal chains are established as  essential for explanations, but they must also be short. The question of which relations to emphasise remains open, as well as the way  in which they can be used to define contrastive explanations, requiring to be able to reason about counterfactual scenarios. 

%% file: 0.4_bibliographie.tex
\bibliographystyle{kr}
\bibliography{kr_bib}

%% file: main.bbl
\begin{thebibliography}{}

\bibitem[\protect\citeauthoryear{Andreas and Guenther}{2021}]{andreas_regularity_2021}
Andreas, H., and Guenther, M.
\newblock 2021.
\newblock Regularity and {Inferential} {Theories} of {Causation}.
\newblock In {\em The {Stanford} {Encyclopedia} of {Philosophy}}. Stanford University.

\bibitem[\protect\citeauthoryear{Beckers}{2021}]{beckers_causal_2021}
Beckers, S.
\newblock 2021.
\newblock Causal sufficiency and actual causation.
\newblock {\em J. Philos. Log.} 50(6):1341--1374.

\bibitem[\protect\citeauthoryear{de Saint-Cyr \bgroup et al\mbox.\egroup }{2016}]{de2016argumentation}
de~Saint-Cyr, F.~D.; Bisquert, P.; Cayrol, C.; and Lagasquie-Schiex, M.-C.
\newblock 2016.
\newblock Argumentation update in yalla (yet another logic language for argumentation).
\newblock {\em International Journal of Approximate Reasoning} 75:57--92.

\bibitem[\protect\citeauthoryear{Doutre, Maffre, and McBurney}{2017}]{doutre2017dynamic}
Doutre, S.; Maffre, F.; and McBurney, P.
\newblock 2017.
\newblock A dynamic logic framework for abstract argumentation: adding and removing arguments.
\newblock In {\em Advances in Artificial Intelligence: From Theory to Practice: 30th International Conference on Industrial Engineering and Other Applications of Applied Intelligent Systems, IEA/AIE 2017}.
\newblock Springer.

\bibitem[\protect\citeauthoryear{Dung}{1995}]{dung_acceptability_1995}
Dung, P.~M.
\newblock 1995.
\newblock On the acceptability of arguments and its fundamental role in nonmonotonic reasoning, logic programming and n-person games.
\newblock {\em Artificial {Intelligence}} 77:321--357.

\bibitem[\protect\citeauthoryear{Fan and Toni}{2015}]{fan2015explanations}
Fan, X., and Toni, F.
\newblock 2015.
\newblock On explanations for non-acceptable arguments.
\newblock In {\em Theory and Applications of Formal Argumentation: 3rd Int. Workshop},  112--127.
\newblock Springer.

\bibitem[\protect\citeauthoryear{Fox and Long}{2006}]{fox_modelling_2006}
Fox, M., and Long, D.
\newblock 2006.
\newblock Modelling {Mixed} {Discrete}-{Continuous} {Domains} for {Planning}.
\newblock {\em Journal of Artificial Intelligence Research} 27:235--297.

\bibitem[\protect\citeauthoryear{Giunchiglia and Lifschitz}{1998}]{giunchiglia_action_1998}
Giunchiglia, E., and Lifschitz, V.
\newblock 1998.
\newblock An {Action} {Language} {Based} on {Causal} {Explanation}: {Preliminary} {Report}.
\newblock In {\em Proc. of the 15th {Nat.} {Conf.} on {Artificial} {Intelligence} and 10th {Innovative} {Applications} of {Artificial} {Intelligence} {Conf.}, {AAAI} 98, {IAAI} 98},  623--630.

\bibitem[\protect\citeauthoryear{Halpern and Pearl}{2005}]{halpern2005causes}
Halpern, J.~Y., and Pearl, J.
\newblock 2005.
\newblock Causes and explanations: A structural-model approach. part i: Causes.
\newblock {\em The British J. Philosophy of Science}.

\bibitem[\protect\citeauthoryear{Lippi and Torroni}{2016}]{lippi2016argumentation}
Lippi, M., and Torroni, P.
\newblock 2016.
\newblock Argumentation mining: State of the art and emerging trends.
\newblock {\em ACM Trans. on Internet Technology} 16(2):1--25.

\bibitem[\protect\citeauthoryear{Menzies and Beebee}{2020}]{menzies_counterfactual_2020}
Menzies, P., and Beebee, H.
\newblock 2020.
\newblock Counterfactual {Theories} of {Causation}.
\newblock In {\em The {Stanford} {Encyclopedia} of {Philosophy}}. Metaphysics Research Lab, Stanford University.

\bibitem[\protect\citeauthoryear{Miller}{2019}]{miller_explanation_2018}
Miller, T.
\newblock 2019.
\newblock Explanation in {Artificial} {Intelligence}: {Insights} from the {Social} {Sciences}.
\newblock {\em Artificial Intelligence} 267:1--38.

\bibitem[\protect\citeauthoryear{Munro \bgroup et al\mbox.\egroup }{2022}]{munro2022argumentation}
Munro, Y.; Bloch, I.; Chetouani, M.; Lesot, M.-J.; and Pelachaud, C.
\newblock 2022.
\newblock Argumentation and causal models in human-machine interaction: A round trip.
\newblock In {\em 8th Int. Workshop on Artificial Intelligence and Cognition}.

\bibitem[\protect\citeauthoryear{Rahwan and Simari}{2009}]{rahwan2009argumentation}
Rahwan, I., and Simari, G.~R.
\newblock 2009.
\newblock {\em Argumentation in artificial intelligence}, volume~47.
\newblock Springer.

\bibitem[\protect\citeauthoryear{Sarmiento \bgroup et al\mbox.\egroup }{2022}]{sarmiento_action_2022-1}
Sarmiento, C.; Bourgne, G.; Inoue, K.; and Ganascia, J.-G.
\newblock 2022.
\newblock Action {Languages} {Based} {Actual} {Causality} in {Decision} {Making} {Contexts}.
\newblock In {\em {PRIMA} 2022, {Proc.}}, LNCS.
\newblock Springer.

\bibitem[\protect\citeauthoryear{Sarmiento \bgroup et al\mbox.\egroup }{2023}]{sarmiento_action_2023}
Sarmiento, C.; Bourgne, G.; Inoue, K.; Cavalli, D.; and Ganascia, J.-G.
\newblock 2023.
\newblock Action {Languages} {Based} {Actual} {Causality} for {Computational} {Ethics:} a {Sound} and {Complete} {Implementation} in {ASP}.
\newblock {\em submitted}.

\bibitem[\protect\citeauthoryear{Wright}{1985}]{wright_causation_1985}
Wright, R.~W.
\newblock 1985.
\newblock Causation in {Tort} {Law}.
\newblock {\em California Law Review} 73(6):1735--1828.

\bibitem[\protect\citeauthoryear{Čyras \bgroup et al\mbox.\egroup }{2021}]{cyras_survey}
Čyras, K.; Rago, A.; Albini, E.; Baroni, P.; and Toni, F.
\newblock 2021.
\newblock Argumentative {XAI}: A survey.
\newblock In {\em {IJCAI-21}},  4392--4399.

\end{thebibliography}
